\documentclass[twoside,11pt]{article}

\usepackage{jmlr2e_mod}

\usepackage{enumitem}

\usepackage{times}
\usepackage[utf8]{inputenc} % allow utf-8 input
\usepackage[T1]{fontenc}    % use 8-bit T1 fonts
\usepackage[usenames,dvipsnames]{color}
\usepackage{hyperref,trimclip,stackengine}
\usepackage{url}            % simple URL typesetting
\usepackage{booktabs}       % professional-quality tables
\usepackage{amsmath,amssymb,amsfonts}
\usepackage{commath}
\usepackage{MnSymbol}
\usepackage{nicefrac}       % compact symbols for 1/2, etc.
\usepackage{microtype}      % microtypography

\usepackage[makeindex]{glossaries}

\usepackage{tikz}
\newcommand*\circled[1]{\tikz[baseline=(char.base)]{
            \node[shape=circle,draw,inner sep=0.5pt] (char) {#1};}}

\DeclareMathOperator{\E}{\mathsf{E}}

\DeclareMathOperator*{\pipes}{\|}
\DeclareMathOperator*{\argmin}{arg\,min}
\DeclareMathOperator*{\esssup}{ess\,sup}

\newcommand{\bernc}{B} % multiplicative constant for Bernstein condition
\newcommand{\reals}{\mathbb{R}}

\newcommand{\ml}{\hat{f}^{\text{\sc ml}}}

\newcommand{\F}{\mathcal{F}}
\newcommand{\G}{\mathcal{G}}
\newcommand{\cH}{\mathcal{H}}
\newcommand{\cF}{\mathcal{F}}

\newcommand{\cK}{\mathcal{K}}

\newcommand{\cS}{\mathcal{S}}
\newcommand{\X}{\mathcal{X}}
\newcommand{\Y}{\mathcal{Y}}
\newcommand{\Z}{\mathcal{Z}}
\newcommand{\cZ}{\mathcal{Z}}
\newcommand{\cY}{\mathcal{Y}}
\newcommand{\cX}{\mathcal{X}}
\newcommand{\cQ}{\mathcal{Q}}

\newcommand{\N}{\mathcal{N}}
\newcommand{\M}{\mathcal{M}}

\newcommand{\rad}{\mathcal{R}}

\newcommand{\ann}[1]{\textsc{ann}(#1)}
\newcommand{\Expann}[1]{\E^{\textsc{ann},#1}}

\newcommand{\Zn}{{Z}^n}

\newcommand{\Zpn}{{\bar{Z}}^n}

\newcommand{\zn}{{z}^n}

\newcommand{\Zp}{\bar{Z}}

\newcommand{\loss}{\ell}

\newcommand{\xsloss}[1]{{R}_{#1}}

\DeclareMathOperator{\KL}{KL}

\newcommand{\shtark}{\text{\sc S}}
\newcommand{\compnew}{\text{\sc comp}}
\newcommand{\fcompnew}{\text{\sc comp}^{\text{\sc full}}}

\newcommand{\rv}[1]{\underline{#1}}

\newcommand{\stochleq}{\ensuremath{\leqclosed}}

\newcommand{\lip}{L} % L_\loss

\newcommand{\Fpart}[1]{\F_{\varepsilon,#1}}

\makeatletter
\DeclareRobustCommand{\qed}{%
  \ifmmode % if math mode, assume display: omit penalty etc.
  \else \leavevmode\unskip\penalty9999 \hbox{}\nobreak\hfill
  \fi
  \quad\hbox{\qedsymbol}}
\newcommand{\qedsymbol}{\BlackBox}
\newenvironment{Proof}[1][\proofname]{\par
  \normalfont
  \topsep6\p@\@plus6\p@ \trivlist
  \item[\hskip\labelsep\bfseries
    #1]\ignorespaces
}{%
  \qed\endtrivlist
}
\newcommand{\proofname}{Proof}
\makeatother
\newglossaryentry{losses}{
  name={\textbf{Losses}},
  description={\nopostdesc},
  sort=A,
  nonumberlist
}

\newglossaryentry{excessLoss}
{
  name={$\xsloss{f}(z)$},
  description={Excess loss, $\loss_f(z) - \loss_{f^*}(z)$},
  sort=A1
}

\newglossaryentry{cumulativeExcessLoss}
{
  name={$\xsloss{f}(\zn)$},
  description={Cumulative excess loss, $\loss_f(\zn) - \loss_{f^*}(\zn)$} \vspace{2mm},
  sort=A2
}

\newglossaryentry{entropified}
{
  name={$q_f(z )$},
  description={Entropified version of $f$, $\frac{p(z) \cdot e^{-\eta \xsloss{f}(z)}}{\E_{Z \sim P } \left[ e^{-\eta \xsloss{f}(Z)} \right]}$},
  sort=A3
}

\newglossaryentry{notation}{
  name={\textbf{Notation}},
  description={\nopostdesc},
  sort=B,
  nonumberlist
}
 
\newglossaryentry{ESI}
{
  name={$U \stochleq_\eta U'$},
  description={Exponential stochastic inequality, 
    $\E_{U, U' \sim P} \left[e^{\eta (U- U')} \right] \leq 1$} \vspace{1mm},
  sort=B1
}

\newglossaryentry{annealedExpectation}
{
  name={$\Expann{\eta} \left[U \right]$},
  description={Annealed expectation, 
                       $-\frac{1}{\eta} \log \E \left[e^{-\eta U} \right]$},
  sort=B2
}

\newglossaryentry{complexities}{
  name={\textbf{Complexities}},
  description={\nopostdesc},
  sort=C,
  nonumberlist
}

\longnewglossaryentry{normConstant}
{
  name={$C(f)$},
  description={Normalization constant for Shtarkov integral, 
                       $\E_{\Zn \sim P } \left[ e^{-\eta \xsloss{f}(\Zn)} \right]$} \vspace{1mm},
  sort=C1
}

\longnewglossaryentry{shtarkovDet}
{
  name={$\shtark(\cF; \hat{f})$},
  description={Shtarkov integral (deterministic version), 
                       $\E_{\Zn \sim P} \left[ 
                             \frac{e^{-\eta  \xsloss{\hat{f}_{ \mid \Zn}}(\Zn) } }{C({\hat{f}_{ \mid \Zn}})}
                         \right]$} \vspace{1mm},
  sort=C21
}

\longnewglossaryentry{maxShtarkovDet}
{
  name={$\shtark(\cF)$},
  description={Maximal Shtarkov integral (deterministic version),

                       $\int_{\cZ^n} \sup_{f \in \cF} q_f(z^n) d \nu(z^n)$} \vspace{1mm},
  sort=C22
}

\longnewglossaryentry{genShtarkovDet}
{
  name={$\shtark(\cF; \hat{f}, w)$},
  description={Generalized Shtarkov integral (deterministic version)

                       $\E_{\Zn \sim P} \left[ 
        \frac{e^{-\eta  \xsloss{\hat{f}_{ \mid \Zn}}(\Zn) } }{C({\hat{f}_{ \mid \Zn}})}
  \cdot w(Z^n)  \right] 
= \int_{\Z^n} q_{\hat{f}_{ \mid \zn}}(\zn) w(z^n) d \nu(\zn)$} \vspace{1mm},
  sort=C23
}

\longnewglossaryentry{genShtarkov}
{
  name={$\shtark(\cF, \hat{\Pi}, w)$},
  description={Generalized Shtarkov integral,

                       $\E_{\Zn \sim P} \left[
                           \exp\left(- \E_{\rv{f} \sim \hat{\Pi} \mid Z^n} \left[
                               \eta \xsloss{\rv{f}}(Z^n) + \log C(\rv{f}) -  \log w(Z^n, \rv{f}) \right] \right) 
                           \right]$} \vspace{1mm},
  sort=C24
}

\newglossaryentry{simpleComp}
{
  name={$\compnew(\cF, \hat{f})$},
  description={Complexity, $\eta^{-1} \log \shtark(\cF, \hat{f})$},
  sort=C31
}

\longnewglossaryentry{genCompDet}
{
  name={$\compnew(\cF, \hat{f}, w, \zn)$},
  description={Generalized complexity (deterministic version),

                       $\frac{1}{\eta} \left(- \log w(z^n) +  \log \shtark(\cF, \hat{f} , w) \right)$} \vspace{1mm},
  sort=C32
}

\newglossaryentry{maxComp}
{
  name={$\compnew(\cF)$},
  description={Maximal complexity, $\eta^{-1} \log \shtark(\cF) = \sup_{\hat{f}} 
\compnew(\cF,\hat{f})$},
  sort=C33
}

\longnewglossaryentry{genComp}
{
  name={$\compnew(\cF, \hat{\Pi}, w, z^n)$},
  description={Generalized complexity,

                       $\frac{1}{\eta} \cdot \left(  
                              \E_{\rv{f} \sim \hat{\Pi} \mid \zn} \left[ - \log w(z^n,\rv{f}) \right] 
                              + \log \shtark(\cF, \hat{\Pi}, w) 
                         \right)$},
  sort=C34
}

\longnewglossaryentry{compFull}
{
  name={$\fcompnew(\cF, \hat{\Pi}, w, z^n)$},
  description={Full generalized complexity,

                       $\compnew(\cF, \hat{\Pi}, w, z^n)  + \E_{\rv{f} \sim \hat{\Pi} \mid z^n}[R_{\rv{f}}(\zn)]$},
  sort=C35
}

\newglossaryentry{Tn}
{
  name={$T_n$},
  description={{\footnotesize $\sup_{f \in \cF} \left\{ 
    \sum_{j=1}^n \left( \loss_{f_0}(Z_j) - \loss_f(Z_j) \right)
    - \E_{\Zn \sim Q_{f_0}} \left[ \sum_{j=1}^n \left( \loss_{f_0}(Z_j) - \loss_f(Z_j) \right) \right] 
\right\}$}},
  sort=C41
}

\newglossaryentry{HLocalComplexity}
{
  name={$\E_{\Zn \sim Q_{f_0}} [ T_n ]$},
  description={$H$-local complexity},
  sort=C42
}

\newglossaryentry{coveringNumber}
{
  name={$\N(\cH, \|\cdot\|, \varepsilon)$},
  description={$\varepsilon$-covering number for $\cH$ in the norm $\|\cdot\|$},
  sort=C51
}

\newglossaryentry{bracketingNumber}
{
  name={$\N_{[\cdot]}(\cH, \|\cdot\|, \varepsilon)$},
  description={$\varepsilon$-bracketing number for $\cH$ in the norm$\|\cdot\|$} \vspace{1mm},
  sort=C52
}

\newglossaryentry{empRadComp}
{
  name={$\rad_n(\cH \mid S_1, \ldots, S_n)$},
  description={Empirical Rademacher complexity, $\E_{\epsilon_1, \ldots, \epsilon_n} \left[ 
       \sup_{h \in \cH} \left| \frac{1}{n} \sum_{i=1}^n \epsilon_i h(S_i) \right| 
   \right]$} \vspace{1mm},
  sort=C53
}

\newglossaryentry{radComp}
{
  name=$\rad_n(\cH)$,
  description={Rademacher complexity, $\E \left[ 
    \sup_{h \in \cH} \left| \frac{1}{n} \sum_{i=1}^n \epsilon_i h(S_i) \right| 
    \right]$},
  sort=C54
}

\makeglossaries

\ShortHeadings{PAC-Bayes--Rademacher--Shtarkov--MDL}{Gr\"unwald Mehta}
\firstpageno{1}
\begin{document}

\title{A Tight Excess Risk Bound via a Unified PAC-Bayesian--Rademacher--Shtarkov--MDL Complexity}

\author{\name Peter D. Gr\"unwald \email pdg@cwi.nl \\
       \addr Centrum Wiskunde \& Informatica and Leiden University
       \AND
       \name Nishant A. Mehta \email nmehta@uvic.ca \\
       \addr University of Victoria}

\maketitle

\begin{abstract}%   <- trailing '%' for backward compatibility of .sty file
  We present a novel notion of complexity that interpolates between
  and generalizes some classic existing complexity notions in learning
  theory: for estimators like empirical risk minimization (ERM) with
  arbitrary bounded losses, it is upper bounded in terms of
  data-independent Rademacher complexity; for generalized Bayesian
  estimators, it is upper bounded by the data-dependent information
  complexity (also known as stochastic or PAC-Bayesian,
  $\KL(\text{posterior} \pipes \text{prior})$ complexity. For (penalized)
  ERM, the new complexity reduces to (generalized) normalized maximum
  likelihood (NML) complexity, i.e.~a minimax log-loss
  individual-sequence regret.  Our first main result bounds excess
  risk in terms of the new complexity. Our second main result links
  the new complexity via Rademacher complexity to $L_2(P)$ entropy, 
  thereby generalizing earlier results of Opper, Haussler,
  Lugosi, and Cesa-Bianchi who did the log-loss case with $L_\infty$.
  Together, these results recover optimal bounds for VC- and large
  (polynomial entropy) classes, replacing localized Rademacher
  complexity by a simpler analysis which almost completely
  separates the two aspects that determine the achievable rates: 
  `easiness' (Bernstein) conditions and model complexity.
 \end{abstract}

\section{Introduction}
We simultaneously address four questions of learning theory:
\begin{enumerate}[label=(\Alph*)]
\item We establish a precise relation between Rademacher complexities
for arbitrary bounded losses and the minimax cumulative log-loss
regret, also known as the \emph{Shtarkov integral} and
\emph{normalized maximum likelihood (NML) complexity}.
\item We bound this minimax regret in terms of $L_2$ entropy. Past results were based on $L_{\infty}$ entropy.
\item We introduce a new type of complexity that enables a unification of data-dependent PAC-Bayesian and
empirical-process-type excess risk bounds into a single clean bound; this bound recovers minimax optimal rates for large classes under Bernstein `easiness' conditions.
\item We extend the link between
excess risk bounds for arbitrary losses and codelengths of Bayesian
codes to general codes.
\end{enumerate}

All four results are part of the chain of bounds in Figure~\ref{fig:bounds}. 
The $\leftarrow$ arrow stands for `bounded in terms of'; the precise bounds (which may hold in probability and expectation or may even be an equality) are given in the respective results in the paper. Red arrows indicate results that are new. 
\begin{figure}[th]
  \stackinset{l}{59.5pt}{t}{31pt}{\hyperref[cor:risk-comp-esi]{\makebox(37,7){}}}{%
  \stackinset{l}{59.5pt}{t}{41pt}{\hyperref[lemma:kl-renyi]{\makebox(30,5){}}}{%
  \stackinset{l}{174pt}{t}{32.5pt}{\hyperref[thm:first]{\makebox(35,5.5){}}}{%
  \stackinset{l}{331pt}{t}{16pt}{\hyperref[prop:newgeneration]{\makebox(43.5,6.5){}}}{%
  \stackinset{l}{252.5pt}{t}{124.5pt}{\hyperref[thm:opper-haussler-talagrand]{\makebox(39,5){}}}{%
  \stackinset{l}{252.5pt}{t}{133.5pt}{\hyperref[cor:opper-haussler-talagrand]{\makebox(40,6.5){}}}{%
  \stackinset{l}{252.5pt}{t}{182.5pt}{\hyperref[thm:opper-haussler-talagrand]{\makebox(39,5){}}}{%
  \stackinset{l}{252.5pt}{t}{191.5pt}{\hyperref[cor:opper-haussler-talagrand]{\makebox(40,6.5){}}}{%
  \stackinset{l}{316pt}{t}{165.5pt}{\hyperref[thm:small-Esup]{\makebox(38.5,5.5){}}}{%
  \stackinset{l}{301pt}{t}{220pt}{\hyperref[thm:small-rad]{\makebox(39,5.5){}}}{%
\includegraphics[width=\textwidth]{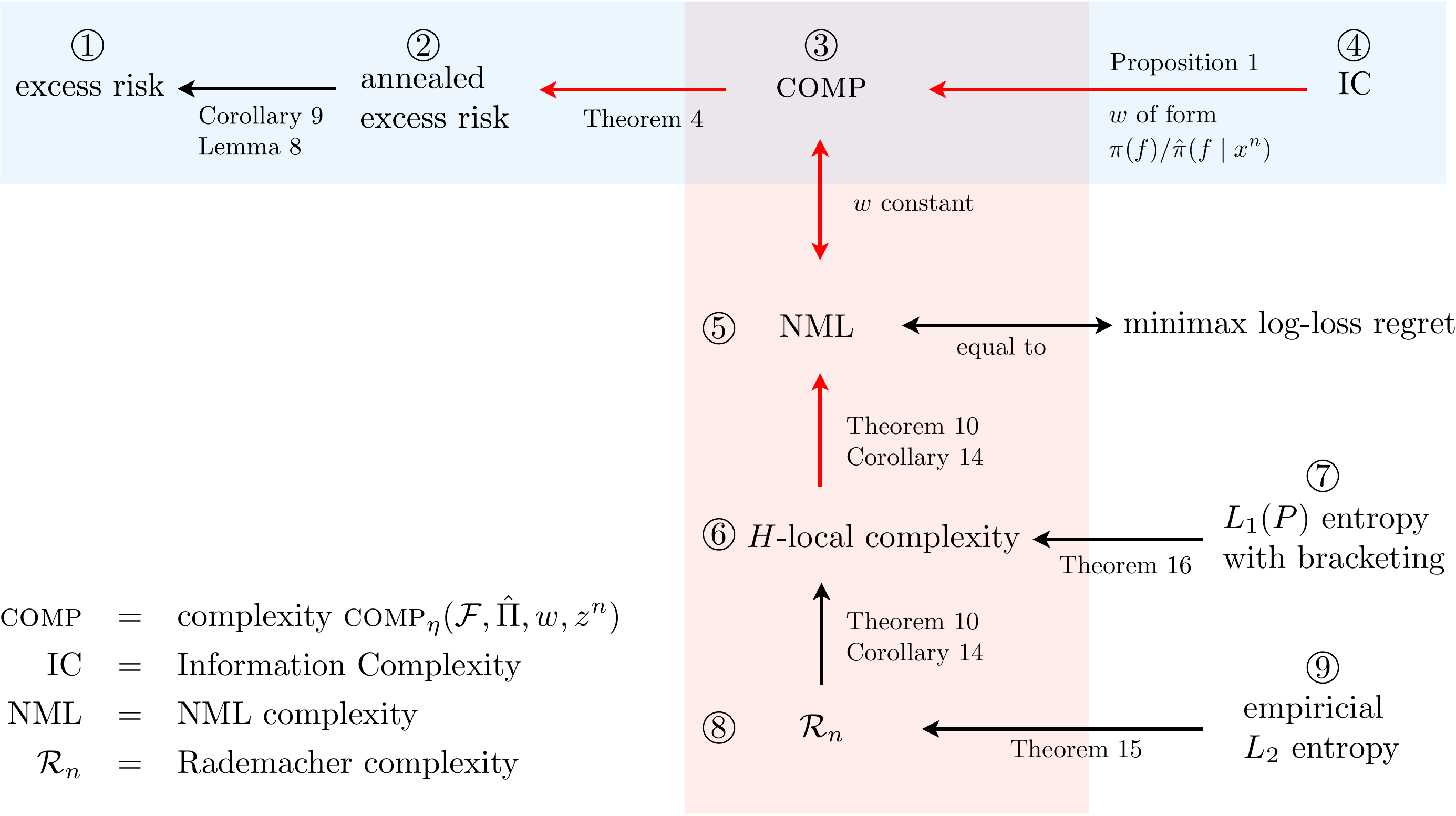}}}}}}}}}}}
\caption{\label{fig:bounds} The tree of bounds we provide; red arrows indicate new results.}
\end{figure}
We start with a family of predictors $\cF$ for an arbitrary \emph{loss function} $\loss$, which, for example, may be log-loss, squared
error loss or $0/1$-loss, and an estimator $\hat{\Pi}$ which on each sample $Z^n = Z_1, \ldots, Z_n$ outputs a distribution $\hat{\Pi} \mid Z^n$ on $\cF$; classic deterministic estimators $\hat{f}$ such as
ERM are represented by taking a $\hat{\Pi}$ that outputs the Dirac measure on $\hat{f}$. 
The main bound 
$\circled{2} \leftarrow \circled{3}$, 
Theorem~\ref{thm:first},
bounds the \emph{annealed excess risk} of a fixed but arbitrary
estimator $\hat{\Pi}$ in terms of its empirical risk on the training
data $Z^n$ plus a novel notion of complexity,
$\compnew_{\eta}(\cF, \hat{\Pi}, w, Z^n)$ (formulas for $\compnew_{\eta}$ and all other concepts in the paper are summarized in the Glossary on page~\pageref{glos:sary}).  The annealed excess risk
is a proxy (and lower bound) of the actual excess risk, the expected
loss difference between predicting with $\hat{\Pi} \mid Z^n$ and
predicting with the actual risk minimizer $f^*$ within $\cF$. 
The bound 
$\circled{1} \leftarrow \circled{2}$ 
(Corollary~\ref{cor:risk-comp-esi}, based on Lemma~\ref{lemma:kl-renyi}, itself from \cite{grunwald2012safe}) 
bounds the actual excess risk in terms of the annealed excess risk, so that we get a true excess risk bound for $\hat\Pi$. 
The complexity $\compnew_{\eta}(\cF, \hat{\Pi}, w, Z^n)$ is
dependent on a \emph{luckiness function}
$w: \cZ^n \times \cF \rightarrow \reals^+_0$; $w$ can be chosen
freely; different choices lead to different complexities and excess
risk bounds. 
For nonconstant $w$, the complexity becomes data-dependent; 
in particular, for $w$ of the form $\pi(f) / \hat{\pi}(f \mid z^n)$, 
where $\pi$ is the density of a `prior' distribution $\Pi$ on $\cF$, 
the complexity becomes, by 
$\circled{3} \leftarrow \circled{4}$ 
(Proposition~\ref{prop:newgeneration}) (strictly) upper bounded by the \emph{information complexity} of \cite{zhang2006epsilon,zhang2006information}, involving a Kullback-Leibler (KL) divergence term $\KL(\hat{\Pi} |Z^n \pipes \Pi)$. 
Information complexity generalizes earlier complexity notions and accompanying bounds from the
information theory literature such as (extended) \emph{stochastic
  complexity} \citep{rissanen1989stochastic,yamanishi1998decision},
\emph{resolvability} \citep{barron1991minimum,BarronRY98}, and also
excess risk bounds from the PAC-Bayesian literature
\citep{audibert2004PAC,catoni2007PAC}. 
Together,
$\circled{1} \leftarrow \circled{2} \leftarrow \circled{3} \leftarrow \circled{4}$ 
recover and strengthen Zhang's bounds.

For constant $w$, the complexity is independent of the data and turns out 
($\circled{3} = \circled{5}$), 
Section~\ref{sec:simpleshtarkov}) to be equal to the
\emph{minimax cumulative individual sequence regret} for sequential
prediction with log-loss relative to a family $\cQ_{\cF}$ of
probability measures defined in terms of $\cF$, also known as the
log-\emph{Shtarkov integral} or \emph{NML (Normalized Maximum
  Likelihood) complexity}. NML complexity has been much studied in the
MDL (minimum description length) literature
\citep{rissanen1996fisher,grunwald2007the}.

\paragraph{Problem A: NML and Rademacher} 
NML complexity can itself be bounded in terms of a new complexity we introduce, $H$-local complexity, which is further bounded in terms of Rademacher complexity $\rad_n$ (Theorem~\ref{thm:opper-haussler-talagrand} and Corollary~\ref{cor:opper-haussler-talagrand}, 
$\circled{5} \leftarrow \circled{6} \leftarrow \circled{8}$). 
Both Rademacher and NML complexities are used as penalties in model selection 
(albeit with different motivations), 
and the close conceptual similarity between NML and Rademacher complexity has been noted by several authors 
(e.g.~\cite{grunwald2007the,zhu2009human,roos2016informal}). For example, as shown by \citet[Open Problem 19, page 583]{grunwald2007the} in classification problems,  
both the empirical Rademacher complexity for 0-1 loss and the
NML complexity of a family of conditional distributions can be simply expressed in terms of a (possibly transformed) minimized cumulative loss that is uniformly averaged over all possible values of the data to be predicted, thereby measuring how well the model $\cF$ can fit random noise.
Theorem~\ref{thm:opper-haussler-talagrand} and
Corollary~\ref{cor:opper-haussler-talagrand} establish, for the first
time, a precise and tight link between NML and Rademacher complexity. The proofs extend a technique due to
\cite{opper1999worst}, who bound NML complexity in terms of
$L_{\infty}$ entropy using an empirical process result of
\cite{yurinskiui1976exponential}.  By using Talagrand's inequality
instead, we get a bound in terms of Rademacher complexity.

\paragraph{Problem B: Bounding NML Complexity with $L_2(P)$ entropy and empirical $L_2$ Entropy} 
If $\cF$ is of VC-type or a class of polynomial empirical $L_2$ entropy, 
the Rademacher complexity can be further bounded, 
(Theorem~\ref{thm:small-rad}, $\circled{8} \leftarrow \circled{9}$), 
in terms of the empirical $L_2$ entropy; 
if $\cF$ admits polynomial $L_1(P)$ entropy with bracketing, 
then $T_n$ is further bounded, (Theorem~\ref{thm:small-Esup}, 
$\circled{6} \leftarrow \circled{7}$), 
in terms of this $L_1(P)$ entropy with bracketing. 
These latter two results are well-known, due to \cite{koltchinskii2011oracle} and
\cite{massart2006risk} respectively, but \emph{in conjunction} with 
$\circled{5} \leftarrow \circled{6}$ 
they become of significant interest for log-loss individual sequence prediction. 
Whereas previous bounds on minimax log-loss regret were invariably in terms of
$L_{\infty}$ entropy 
\citep{opper1999worst,cesa2001worst,rakhlin2015sequential}, 
the aforementiond two results allow us to obtain bounds in terms of $L_1(P)$ entropy and empirical $L_2$ entropy, where $P$ can be any member of the class $\cQ_{f}$. Unlike the latter two
works, however, our results are restricted to static experts that treat
the data as i.i.d.

\paragraph{Problem C: Unifying data-dependent and empirical process-type excess risk bounds}
As lamented by Audibert (\citeyear{audibert2004PAC,audibert2009fast}),
despite their considerable appeal, standard PAC-Bayesian and KL excess
risk bounds do not lead to the right rates for large classes, 
i.e.~with polynomial $L_2(P)$ entropy. On the other hand, standard
Rademacher complexity generalization and excess risk bound analyses
are not easily extendable to either penalized estimators or
generalized Bayesian estimators that are based on updating a prior distribution; also handling logarithmic loss appears difficult. 
Yet 
$\circled{1} \leftarrow \circled{2} \leftarrow \circled{3}$ 
shows that there does
exist a single bound capturing all these applications --- by varying
the function $w$ one can get both (a strict strengthening of) the KL
bounds and a Rademacher complexity-type excess risk bound. 
In this way, via the chain of bounds 
$\circled{1} \leftarrow \ldots \leftarrow \circled{7}/\circled{9}$, 
we recover rates for
empirical risk minimization (ERM) that either are minimax optimal (for
classification) or the best known rates for ERM (for other losses),
even for VC and polynomial entropy classes; the rates depend in the
right way on the `easiness' of the problem as measured by the \emph{central condition} \Citep{vanerven2015fast}, which generalizes Tsybakov's (\citeyear{tsybakov2004optimal}) margin condition and Bernstein conditions \citep{bartlett2005local}.

\paragraph{Problem D:  Excess Risk Bounds and Data Compression}
{While Zhang's bound holds for arbitrary `posteriors' $\hat{\Pi}$, one
gets the best excess risk bounds if one takes $\hat{\Pi}$ to be a
\emph{generalized Bayesian estimator}. With such a $\hat{\Pi}$,   the information
complexity can be expressed in terms of a (generalization of) the
cumulative log-loss of a Bayesian sequential prediction strategy
\citep{zhang2006epsilon,grunwald2012safe} defined relative to the
constructed probability model $\cQ_{\cF}$.} By the correspondence
between codelengths and cumulative log-loss (reviewed in
Section~\ref{sec:intro-complexity}), we may say that we bound an 
\emph{excess risk in terms of a codelength}. 
$\circled{3} = \circled{5}$ 
shows that we
also get a useful excess risk bound in terms of the codelengths of the
minimax (NML) code --- interestingly, Bayes and NML codes are the two
central codes in the \emph{universal coding} literature
\citep{BarronRY98,grunwald2007the}. In fact, our work shows that the
correspondence between excess risk bounds and codelengths is a quite
general phenomenon, not particular to Bayes and NML codes: in
Section~\ref{sec:intro-complexity} we show that there is a 1-to-1
relation between luckiness functions $w$ and codes for data based on
$\cQ_f$: each luckiness function $w$ defines, up to scaling, a
different code giving rise to a different complexity
$\compnew_{\eta}(\cF,\hat{\Pi},w,z^n)$ and hence a different excess
risk bound, and vice versa.  Because of its relation to data
compression, it becomes easy to extend the approach to model selection
by using `two-part codes' (Section~\ref{sec:twopart}).

For yet other choices of $w$, we obtain bounds for penalized ERM with
arbitrary bounded penalization functions
(Section~\ref{sec:penalized}); if we specialize to log-loss, the
complexity bound becomes a minimax regret-with-luckiness-term as
considered in recent papers on sequential log-loss prediction such as
\citep{kakade2006worst,bartlett2013horizon}. Many other choices of $w$ are possible and remain subject for future investigation. 

\paragraph{Additional Features and Limitations}
The full story above can only be told for bounded losses, although the
bounds 
$\circled{2} \leftarrow \circled{3}, \circled{3} \leftarrow \circled{4}$ 
also hold
for unbounded losses, and 
$\circled{1} \leftarrow \circled{2}$ 
was recently
extended to unbounded losses under a mild additional condition
\citep{grunwald2016fast}. Remarkably, 
$\circled{2} \leftarrow \circled{3}$ 
is as tight as can be: when viewed in terms of exponential moments, it is
really an equality rather than a bound; this suggests that, no matter
the choice of $w$, the resulting bound is essentially
unimprovable. The complexity $\compnew_{\eta}$ depends on a 
\emph{learning rate parameter} $\eta$, and all bounds in the figure
become different depending on the choice of $\eta$.  The optimal
$\eta$ depends on the easiness of the problem at hand, as measured by
central/Bernstein/Tsybakov's conditions (see above). ERM can be
applied (and optimal rates can be obtained) without knowledge of the
optimal $\eta$; to get the right rates for Bayesian and penalized ERM
algorithms however, these algorithms should be made dependent on
$\eta$ ($\eta$ is akin to $1/\lambda$ in the lasso and ridge
regression); in practice, one can learn it from the data using an algorithm such as the `safe Bayes' algorithm of \cite{grunwald2012safe} (Section~\ref{sec:twopart}).

\paragraph{Contents} 
In
Section~\ref{sec:simpleshtarkov}, we introduce the simple
data-independent version of our complexity, $\compnew(\cF, \hat{f})$,
which is really the NML complexity.  In
Section~\ref{sec:generalizedshtarkov} we extend our notion of
complexity to the generalized data-dependent form $\compnew(\cF,
\hat{\Pi}, w, z^n)$.  Section~\ref{sec:zhang-nml} contains our first
main result, Theorem~\ref{thm:first}.  In
Section~\ref{sec:complexity}, we derive our second main result,
Theorem~\ref{thm:opper-haussler-talagrand} and its
Corollary~\ref{cor:opper-haussler-talagrand}, a bound on
$\compnew(\cF, \hat{f})$ in terms of Rademacher complexity; we also
present a concrete application of this result,
Theorem~\ref{thm:excess-risk-erm}, which provides the best known rates
for ERM under Bernstein conditions for bounded loss functions in a
number of situations. Section~\ref{sec:applications} gives various
applications of our result.  Finally, Section~\ref{sec:discussion}
closes with a discussion of our work in the context of other recent
works. All long proofs can be found in the appendix. Mathematical definitions and notations are summarized in the
Glossary on page~\pageref{glos:sary}.

\section{The Novel Complexity Notion}
\label{sec:intro-complexity}

\subsection{Preliminaries}
In the statistical learning problem \citep{vapnik1998statistical}, 
a labeled sample $Z^n = Z_1, \ldots, Z_n$ is drawn independently from probability distribution $P$ over $\Z = (\X \times \Y)$, where, for each $j \in [n]$, we have $Z_j = (X_j, Y_j)$. 
We are given an action space or \emph{model} $\F$ 
and a \emph{loss function} $\loss: \cF \times \cZ \rightarrow \reals$, where we denote the loss that action or predictor $f$ makes on $z$ as $\loss_f(z)$. 
Loss functions such as 0-1, squared error, and log-loss (for joint densities on $z=(x,y)$) can all be expressed this way: 
in the former two cases $\cF$ consists of functions $f: \cX \rightarrow \cY$, with $\loss_f(x,y) = |y-f(x)|$ and $(y- f(x))^2$ respectively, and in the latter case $\cF$ is a set of probability densities on $\cZ = \cX \times \cY$ relative to some underlying measure $\nu$ and $\loss_f(x,y) = - \log f(x,y)$, with $\log$ denoting natural logarithm in this paper. 
An \emph{estimator} or \emph{learning algorithm} $\hat{\Pi}$ is a
function from $\cZ^n$ to distributions over $\cF$. Here and in the
sequel we simply assume that $\cF$ is endowed with a suitable sigma-algebra
$\Sigma$ so that $(\cF,\Sigma)$ is a measurable space and all
functions we refer to are measurable. We will write $\hat{\Pi} \mid
z^n$ to denote the distribution chosen for data $z^n$. In practice
$\hat{\Pi}$ is often supported entirely on a single function $\hat{f}
\in \F$; in that case we simply write the estimator as $\hat{f}$ and
the $f$ chosen for given data $z^n$ as $\hat{f}_{|z^n}$. An example of
such a \emph{deterministic estimator} is ERM, the empirical risk
minimizer. An example of a \emph{randomized estimator} is obtained by
setting $\hat{\Pi} \mid z^n$ to be the generalized $\eta$-Bayesian
posterior \citep{zhang2006information,catoni2007PAC}, which we
explicitly define in Section~\ref{sec:generalizedshtarkov}. Henceforth, we simply call $\hat{\Pi}$
an `estimator' irrespective of whether it is deterministic or randomized.

We aim to learn distributions $\hat\Pi$ that obtain low expected
\emph{risk} $\E_{f \sim {\hat\Pi}} [ \E_{Z \sim P} [ \loss_f(Z) ]
]$. The risk of an action $f$ is the expected value of the loss
$\loss_f(Z)$ suffered when playing action $f$ and the actual outcome
is $Z$.  A natural way to measure the quality of $\hat{\Pi}$ on data
$z^n$ is therefore the \emph{excess risk}
$\E_{f \sim \hat{\Pi}\mid z^n} [ \E_{Z \sim P} [ \loss_f(Z) -
\loss_{f^*}(Z) ] ]$, where $f^*$ is a minimizer of the risk over $\F$; 
like many other authors (e.g \cite{bartlett2005local}) 
we assume throughout this work that such a minimizer exists. 
We use the notation \glsadd{excessLoss}
$\xsloss{f}(z) = \loss_f(z) - \loss_{f^*}(z)$, extended to samples
$\zn = (z_1, \ldots, z_n) \in \Z^n$ as \glsadd{cumulativeExcessLoss}
$\xsloss{f}(\zn) = \sum_{i=1}^n \left( \loss_f(z_i) - \loss_{f^*}(z_i)
\right)$. Note that, when $Z^n \sim P$, $\hat{\Pi} \mid Z^n $ and
$\hat{f}_{|Z^n}$ can be thought of as random variables so we simply
write them without $Z^n$ whenever this cannot cause any confusion.

\subsection{The Novel Complexity Measure, Simple Case}
\label{sec:simpleshtarkov}
To prepare for the definition of our complexity measure $\compnew$, we first need to associate each $f \in \cF$ with an associated 
probability distribution $Q_f$. {We may assume without loss of generality that the underlying distribution $P$ on $Z$ has a density $p$ with respect to some base measure $\nu$ (we could for example take $\nu = P$ but the formulas below are easier to parse for general $\nu$).} Now for each $f \in \F$, we define $Q_f$ to be the distribution over $\cZ$ with density (with respect to the same base measure $\nu$)
\glsadd{entropified}
\begin{align}\label{eq:entropified}
q_f(z ) := \frac{p(z) \cdot e^{-\eta \xsloss{f}(z)}}{\E_{Z \sim P } \left[ e^{-\eta \xsloss{f}(Z)} \right]} .
 \end{align}
We extend the definition to $n$ outcomes by taking the product densities, $q_{f}(z^n) := \prod_{i=1}^n q_{f}(z_i)$. 
In this way the model $\cF$ is itself
mapped to a set $\cQ_{\cF} = \{q_f : f \in \cF \}$ of
probability densities, the mapping depending on the loss function $\loss$ of
interest, but also (suppressed in notation) on $\eta$, $f^*$, and on the `true' $P$; this is an instance of the
`entropification procedure' suggested by \cite{grunwald1999viewing}.

We are now ready to define our new complexity measure. For simplicity
of we first present the very special case of deterministic estimators
$\hat{f}$ without data dependence; this case is sufficient to make the
connection to minimax regret and Rademacher complexities in
Section~\ref{sec:complexity}.  For this setting we define the
\emph{Shtarkov integral} (name to be explained below) as
\glsadd{shtarkovDet}
\begin{multline}\label{eq:shtarkovdet}
\shtark(\cF; \hat{f}) 
:= \E_{\Zn \sim P} \left[ 
        \frac{e^{-\eta  \xsloss{\hat{f}_{ \mid \Zn}}(\Zn) } }{C({\hat{f}_{ \mid \Zn}})}
    \right] 
= \int_{\Z^n} q_{\hat{f}_{ \mid \zn}}(\zn) d \nu(\zn)  \\
\ \ \ldots \left(\   
\overset{\text{(if $\ell$ log-loss, $\eta =1$, model correct)} }{=} 
\int_{\Z^n} p_{\hat{f}_{ \mid \zn}}(\zn) d \nu(\zn)  \right),
\end{multline}
where, for any $f \in \F$, \glsadd{normConstant} $C(f) := \E_{\Zn \sim P } \left[ e^{-\eta \xsloss{f}(\Zn)} \right]$ is the normalization constant. 
Whenever $\shtark(\cF, \hat{f})$ is finite (as will automatically be the case with bounded loss), 
the corresponding \emph{complexity of model $\cF$ equipped with $\hat{f}$} is defined as
\glsadd{simpleComp}
\begin{align}\label{eq:simplecomplexity}
\compnew(\cF, \hat{f}) := \eta^{-1} \log \shtark(\cF, \hat{f}). 
\end{align}

$\compnew, \shtark, q_f$, and normalizer $C$ all depend on $\eta$, but this is suppressed in notation unless needed for clarity.  
The final equality in \eqref{eq:shtarkovdet} holds 
in the very special case that the 
original loss function is log-loss, $\eta= 1$, and $\cF$
contains the density $p$ of $P$ (`the model is correct'). 
In that case  $f^* = p$ (since log-loss is a \emph{proper loss}, see e.g.~\citep{gneiting2007strictly}) 
$C(f)$ evaluates to $1$ for all $f \in \cF$, 
$\mathcal{Q}_{\cF, \loss}$ is equal to $\cF$, 
and $\xsloss{f}(z) = - \log f(z) + \log p(z)$; thus, \eqref{eq:entropified} reduces to $q_f(z) = f(z)$, and the final equation in \eqref{eq:shtarkovdet} follows. We further define the \emph{maximal complexity} $\compnew(\cF)$ as
\glsadd{maxShtarkovDet} \glsadd{maxComp}
\begin{equation}\label{eq:maxcomplexity}
\shtark(\cF) := \int_{\cZ^n} \sup_{f \in \cF} q_f(z^n) d \nu(z^n) \ \ \ ;
\ \ \ \compnew(\cF) := \eta^{-1} \log \shtark(\cF) = \sup_{\hat{f}} 
\compnew(\cF,\hat{f}), 
 \end{equation}
 where the final equality is a trivial consequence of the definition,
 the $\sup$ ranging over all deterministic estimators that can be
 defined on $\cF$.  

We often use the following observation
 due to (e.g.) \cite{opper1999worst}: Let $\cK$ be a finite set and
 let $\{ \cF_k : k \in \cK \}$ be a partition of $\cF$. Then for every
 deterministic estimator,
\begin{align} \label{eq:decomposenml}
\compnew(\cF,\hat{f}) \leq \frac{\log | \cK|}{\eta} + \max_{k \in \cK}
\compnew(\cF_k). 
\end{align}
This result follows as a special case of Proposition~\ref{prop:decomposingcomp} in Section~\ref{sec:applications}, but its proof is simple enough to state in just a few lines:
\begin{align*}
\compnew(\cF,\hat{f}) 
&\leq \eta^{-1} \log \int_{\cZ^n} \max_{k \in \cK} \sup_{f \in \cF_k} q_f(z^n) d \nu(z^n) \\
&\leq \eta^{-1} \log \int_{\cZ^n} \sum_{k \in \cK} \sup_{f \in \cF_k} q_f(z^n) d \nu(z^n) \\ 
&\leq \eta^{-1} \log |\cK| + \eta^{-1} \max_{k \in \cK} \log \int_{\cZ^n} \sup_{f \in \cF_k} q_f(z^n) d \nu(z^n) .
\end{align*}
Using \eqref{eq:decomposenml}, we can link $\compnew$ to Rademacher
complexity, which we will do in Section~\ref{sec:prelloss}
and~\ref{sec:rademacher}. Below, we first link $\compnew$ to log-loss
prediction, extend it to encompass data-dependent and PAC-Bayesian
complexities and present our excess risk bound for the general
complexities (Section~~\ref{sec:prelloss} and~\ref{sec:rademacher})
can be read without this material).
\paragraph{Minimax Cumulative Log-Loss Interpretation of $\compnew$}
For every given estimator $\hat{f}$, we can define a density $r$ on $\cZ^n$ relative to $\nu$ by setting
\begin{equation}\label{eq:nmldist}
r(z^n) := \frac{q_{\hat{f}}(z^n)}{\shtark(\cF,\hat{f})}, 
\end{equation}
which evidently integrates to $1$ and hence is a probability density
(different choice of estimator $\hat{f}$ leads to different $r$; this
is suppressed in the notation). We can use density $r$ to
sequentially predict $Z_1, Z_2, \ldots, Z_n$ by predicting $Z_i$
with the corresponding conditional density $r(Z_i \mid
Z^{i-1})$. The cumulative log-loss obtained this way is given by
\begin{align*}
\sum_{i=1}^n -\log r(Z_i \mid Z^{i-1}) = - \log r(Z^n) ,
\end{align*}
the latter equality following by definition of conditional probability
and telescoping. Because of the correspondence, via Kraft's
inequality, of log-loss prediction and data compression, we can also
think of this quantity as a codelength.  Similarly, $\min_{f \in \cF}
- \log q_f(Z^n)$ is the minimum cumulative loss one could have
obtained \emph{with hindsight}, i.e.~if one had sequentially predicted the
$Z_i$ by the $q_f$ that 
turned out to minimize $- \log q_f$ on $Z^n$.  Assuming this minimum is
well-defined it is of course achieved by $\ml$, the maximum likelihood
estimator relative to $\cQ$, for which evidently also $\compnew(\cF) =
\compnew(\cF,\ml)$. Thus we get that for all $z^n \in \cZ^n$,
\begin{align}\label{eq:cldiff}
\eta^{-1} \cdot \compnew(\cF,\hat{f}) & = 
\log \shtark(\cF,\hat{f}) = - \log r(z^n) - \left(
- \log q_{\hat{f}}(z^n) \right) \nonumber \\
& \overset{\text{if $\hat{f} = \ml$} }{=} 
= \eta^{-1} \cdot \compnew(\cF) = - \log r(z^n) - \min_{f \in \cF} 
\left( - \log q_{f}(z^n) \right),\end{align}
the first equation holding for general $\hat{f}$ and the second for $\ml$. 
The final expression is just  the (cumulative log-loss) \emph{regret} of $r$ on data $z^n$, which, by \eqref{eq:cldiff}, is
constant on $z^n$. As first noted by \cite{shtarkov1987universal},
this implies that \eqref{eq:cldiff} is also the \emph{minimax
  individual-sequence regret} relative to the model $\cQ$ when
sequentially predicting outcomes $Z_1, \ldots, Z_n$ with the log-loss; the corresponding optimal
sequential prediction strategy $r$ is usually called the normalized
maximum likelihood (NML) or Shtarkov density; see
\cite{rissanen1996fisher,grunwald2007the,opper1999worst,cesa2001worst}
for details.
\subsection{Allowing Data-Dependency}
We now generalize the complexity definition above for arbitrary
deterministic $\hat{f}$ so that it becomes data dependent; further
extension to randomized estimators follows in Section~\ref{sec:generalizedshtarkov}. 
The central concept we need is that of a \emph{luckiness function}
$w: \cZ^n \rightarrow \reals^+_0$; every combination of estimator and
luckiness function will, up to scaling, define a unique version of
complexity; and every such complexity will induce a different
data-dependent bound on excess risk. We call $w$ `luckiness function'
since it will influence our excess risk bounds so that they become better
iff we are `lucky' in the sense that $P$ is such that $w(X^n)$ will be
large with high probability).

The \emph{generalized Shtarkov integral} for estimator $\hat{f}$ relative
to luckiness function $w$ is defined as \glsadd{genShtarkovDet}
\begin{align}\label{eq:shtarkovdetb}
\shtark(\cF, \hat{f}, w) 
:= \E_{\Zn \sim P} \left[ 
        \frac{e^{-\eta  \xsloss{\hat{f}_{ \mid \Zn}}(\Zn) } }{C({\hat{f}_{ \mid \Zn}})}
  \cdot w(Z^n)  \right] 
= \int_{\Z^n} q_{\hat{f}_{ \mid \zn}}(\zn) w(z^n) d \nu(\zn), 
\end{align}
and, whenever $\shtark(\cF,\hat{f},w) < \infty$, we define 
the corresponding data-dependent complexity as
\glsadd{genCompDet}
\begin{equation}\label{eq:simplecomplexityb}
\compnew(\cF, \hat{f}, w, \zn) 
:=  \frac{1}{\eta} \left(- \log w(\zn) + \log \shtark(\cF, \hat{f} , w) \right).
\end{equation}
Both expressions evidently  reduce to \eqref{eq:shtarkovdet} and \eqref{eq:simplecomplexity} if we take $w$ constant over $\cZ^n$. 
\paragraph{Cumulative Log-Loss Interpretation}
Fix an arbitrary estimator $\hat{f}$. Then for any luckiness function $w$ 
with $\shtark(\cF,\hat{f},w)< \infty$, we can define the probability density
\begin{equation}\label{eq:nmldistb}
r_w(z^n) := \frac{q_{\hat{f}|z^n}(z^n) \cdot w(z^n)}{\shtark(\cF,\hat{f},w)}, 
\end{equation}
with \eqref{eq:nmldist} being the special case with $w \equiv 1$.
Just as with $r_1$, for general such $w$, $r_w$ can be thought of as a
sequential prediction strategy, and $\eta \cdot \compnew(\cF,\hat{f},w,z^n) - \log q_{\hat{f}_{|z^n}}(z^n) = - \log
r_w(z^n)$ is the cumulative log-loss achieved by $r_w$. Different (up
to scaling) $w$ generate different log-loss prediction strategies
(codes) and corresponding complexities.  Conversely, for every
probability density $r'$ relative to $\nu$ on $Z^n$, we can set a
luckiness measure $w(z^n)$ proportional to
$r'(z^n)/q_{\hat{f}|z^n}(z^n)$; with the appropriately scaled choice of $w$, $r_w$ will
coincide $r'$; we thus have a $1$-to-$1$-correspondence between
luckiness functions $w$ with $\shtark(\cF,\hat{f},w) < \infty$, codes
and complexities.
\subsection{The Novel Complexity Measure, General Case}
\label{sec:generalizedshtarkov}
Here we further  generalize the complexity definition so that it can
output distributions $\hat{\Pi} \mid \Zn$ on $\cF$.  For this we need to extend the domain of the luckiness function to encompass $\cF$, i.e.~we now take arbitrary functions of the form $w: \cZ^n \times \cF \rightarrow \reals^+_0$.

The {generalized Shtarkov integral} for estimator $\hat{\Pi}$ relative to luckiness function $w$ 
is defined as
\glsadd{genShtarkov}
\begin{align}\label{eq:shtarkovgen}
\shtark(\cF, \hat{\Pi}, w) 
:=  \E_{\Zn \sim P} \left[
\exp\left(- \E_{\rv{f} \sim \hat{\Pi} \mid Z^n} \left[\eta 
\xsloss{\rv{f}}(Z^n) + \log C(\rv{f}) -  \log w(Z^n, \rv{f}) \right] \right) \right],
\end{align}
and the generalized (data-dependent) model complexity corresponding to \eqref{eq:shtarkovgen} is now defined as
\glsadd{genComp}
\begin{align}\label{eq:gencomp}
\compnew(\cF, \hat{\Pi}, w, z^n) 
& :=  \frac{1}{\eta} \cdot \left(  
          \E_{\rv{f} \sim \hat{\Pi} \mid \zn} \left[ - \log w(z^n,\rv{f}) \right] 
          + \log \shtark(\cF, \hat{\Pi}, w) \right).
\end{align}
Both expressions are readily seen to generalize
\eqref{eq:shtarkovdetb} and \eqref{eq:simplecomplexityb} respectively:
if, for a given deterministic estimator $\hat{f}$, we take
$\hat{\Pi}(\cdot \mid Z^n)$ to be $\delta_{\hat{f}}$ (the Dirac
measure on $\hat{f}_{\mid \Zn}$) and we take a function $w(z^n,f)
\equiv w(z^n)$ that does not depend on $f$, then the expressions above
simplify trivially to \eqref{eq:shtarkovdetb} and
\eqref{eq:simplecomplexityb} respectively; thus
$\compnew(\cF,\delta_{\hat{f}},w,z^n) = \compnew(\cF,\hat{f},w,z^n)$. 
Finally, we define 
\glsadd{compFull}
\begin{align}\label{eq:compnewfull}
\fcompnew(\cF, \hat{\Pi}, w, z^n) := 
\compnew(\cF, \hat{\Pi}, w, z^n)  + \E_{\rv{f} \sim \hat{\Pi} \mid z^n}[R_{\rv{f}}(\zn)]
\end{align}
as the sum of the complexity and the expected excess loss that a
random draw from $\hat{\Pi}$ achieves on the data.
\paragraph{$\compnew$ generalizes information complexities}
To explain how PAC-Bayesian type complexity arise a special case of
$\compnew$, we consider luckiness measures $w$ that are defined in
terms of probability distributions $\Pi$ on $\cF$ that do not depend
on the data; we call these `priors'. For notational convenience it is
useful to assume (without loss of generality) that $\Pi$ has a
density $\pi$ relative to some underlying measure $\rho$ on $\cF$
and that, for all $z^n \in \cZ^n$, $\hat{\Pi} \mid z^n$ also has a density $\hat{\pi} \mid z^n$ relative to $\rho$.
\begin{proposition}\label{prop:newgeneration}
Consider arbitrary $\Pi$ and $\hat{\Pi}$ as above with densities $\pi$ and $\hat\pi \mid z^n$ relative to some $\rho$. Set $w(z^n,f) := \pi(f) /\hat{\pi}(f \mid z^n)$. Then we have
\begin{align}\label{eq:sbound}
S(\cF,\hat{\Pi},w) \leq 1. 
\end{align}
Consequently,
\begin{align}\label{eq:infcomp}
\fcompnew(\cF, \hat{\Pi}, w, z^n) \leq 
\E_{\rv{f} \sim \hat{\Pi} \mid \zn}[R_{\rv{f}}(\zn)]+ \eta^{-1} \cdot \KL(\; ( \hat{\Pi} \mid \zn) \pipes \Pi \; ),
\end{align}
where $\KL(\; ( \hat{\Pi} \mid \zn) \pipes \Pi \; ) = \E_{\rv{f} \sim \hat{\Pi} \mid \zn}\left[ \log \hat{\pi}(f \mid z^n)/\pi(f) \right] $ is KL divergence.
\end{proposition}
\begin{proof}
By Jensen's inequality applied to \eqref{eq:shtarkovgen}, we have, using the definition of $w$ and Fubini's theorem, 
\begin{align*}
\shtark(\cF, \hat{\Pi}, w) 
& \leq   \E_{\Zn \sim P} \E_{\rv{f} \sim \hat{\Pi} \mid z^n}  \left[
\frac{e^{ - \eta 
\xsloss{\rv{f}}(z^n)}}{C(\rv{f})} \cdot  w(z^n, \rv{f})  \right]
= \E_{\Zn \sim P} \E_{\rv{f} \sim  \Pi} \left[ \frac{e^{ - \eta 
\xsloss{\rv{f}}(z^n)}}{C(\rv{f})}\right] \\ & =
\E_{\rv{f} \sim  \Pi} \left[ \int_{z^n} q_{\rv{f}}(z^n) d \nu(z^n) \right] = 1,
\end{align*} 
which gives \eqref{eq:sbound}; \eqref{eq:infcomp} follows by plugging in our choice of $w$ into the definition of $\compnew$.
\end{proof}

We thus see that $\fcompnew$ is upper bounded by \emph{information complexity} defined relative to prior $\Pi$ 
  \citep{zhang2006epsilon,zhang2006information}, which is just \eqref{eq:infcomp} normalized (divided by $n$). The 
notion of information complexity is also used to bound excess risks in the PAC-Bayesian approach 
of \cite{catoni2007PAC} and \cite{audibert2004PAC}.
As noted by \cite{zhang2006information}, the right-hand side of \eqref{eq:infcomp} is minimized if 
we take as our estimator $\hat{\Pi}$ the $\eta$-generalized Bayesian estimator, \begin{align}\label{eq:posteriorC}
\hat{\pi} (f \mid z^n) := \frac
{\exp\left(-\eta \sum_{i=1 }^{n} \loss_f(z_i)   \right) \cdot \pi(f)}{\int
  \exp\left(-\eta \sum_{i= 1 }^{n} \loss_f(z_i)  \right) \cdot \pi (f) d\rho(f)},
\end{align}
and in that case is equal to the \emph{generalized marginal
  likelihood}, known in the MDL literature as the \emph{extended
  stochastic complexity} \citep{yamanishi1998decision}
\begin{equation}\label{eq:esc} - \eta^{-1} \log \E_{\rv{f} \sim W} [\exp(- \eta \xsloss{\rv{f}}(\Zn))],\end{equation} 
which, for $\eta =1 $ and $\loss$ the log-loss, coincides with the
standard log Bayesian marginal likelihood.

We provide some further simple properties of $\compnew(\cF, \hat{\Pi},w,z^n)$ 
for general $\hat{\Pi}$ and $w$ in Section~\ref{sec:applications}. 
\paragraph{Cumulative Log-Loss Interpretation}
Just as for deterministic estimators, we note that every randomized estimator $\hat{\Pi}$ and luckiness function $w$ defines a probability density/prediction strategy on $Z^n$  by setting 
\begin{align*}
r_{w}(z^n) := 
\frac{
e^{ 
\left( \E_{\rv{f} \sim \hat{\Pi} \mid \zn} \left[  \log q_{\rv{f}}(z^n)\cdot  w(z^n,\rv{f}) \right] \right)}}{S(\cF, \hat{\Pi}, w)} ,
\end{align*}
and just as before, $\compnew$ can be interpreted in terms of the `code' $r_w$.
\section{First Main Result: Bounding Excess Risk in Terms of New Complexity}
\label{sec:zhang-nml}

In this section and Section~\ref{sec:complexity}, we restrict to the bounded loss setting; given this restriction, it is without loss of generality that we assume that
\begin{align}
\sup_{f, g \in \F} \esssup |\loss_f(Z) - \loss_g(Z)| \leq \frac{1}{2} , 
\tag{A1} \label{eqn:bounded}
\end{align}
as this always can be accomplished by an appropriate scaling of the loss.

Before presenting our first main result, it will be useful to
introduce a variant of an ordinary expectation as well as some
notation.  For $\eta >0$ and general random variables $U$, we define
the \glsadd{annealedExpectation} \emph{annealed expectation} (see \cite{grunwald2016fast} for the origin of this terminology) as
\begin{align}\label{eq:genren}
\Expann{\eta} \left[U \right] 
= -\frac{1}{\eta} \log \E \left[e^{-\eta U} \right].
\end{align}

Below we will first bound the annealed excess risk rather than the
standard excess risk and then continue to bound the latter in terms
of the former.  Our first main result below may be expressed
succinctly via the notion of \emph{exponential stochastic inequality}, \glsadd{ESI}
\begin{definition}[Exponential Stochastic Inequality (ESI)] 
\label{def:esi} Let $\eta > 0$ and let $U, U'$ be random variables on some probability space with probability measure $P$. We define
\begin{align}\label{eq:esi}
U \stochleq_\eta U'  
\,\,\,\, \Leftrightarrow \,\,\,\, 
\E_{U, U' \sim P} \left[e^{\eta (U- U')} \right] \leq 1,
\end{align}
and we write $U \stochleq_{\eta}^* U'$ iff the right hand of
\eqref{eq:esi} holds with equality.
\end{definition}
Clearly $U \stochleq_{\eta}^* U' \Rightarrow U \stochleq_{\eta} U'$. An ESI simultaneously captures high probability and in-expectation results:
\begin{proposition}[ESI Implications] \label{prop:drop} 
For all $\eta > 0$, if $U \stochleq_{\eta} U'$ then, 
(i), $\E [ U ] \leq \E [ U' ]$; 
and, (ii), for all $K >0$, with $P$-probability at least $1- e^{-K}$, 
$U \leq U' + K/\eta$. 
\end{proposition} 
\begin{proof} 
  Jensen's inequality yields (i). 
  Apply Markov's inequality to $e^{-\eta (U - U')}$ for (ii).
\end{proof}
We now present  our first main result, a new bound that interpolates between the Zhang bound and standard empirical process theory bounds for handling large classes and that is sharp in the sense that it really is an equality of exponential moments. 
\begin{theorem} \label{thm:first}
For every randomized estimator $\hat{\Pi}$ and every 
luckiness function $w: \cZ^N \times \cF \rightarrow \reals^+_0$, we have
\begin{align}\label{eqn:cond-first}
\E_{\rv{f} \sim \hat{\Pi}_{| \Zn}} \left[ 
\Expann{\eta}_{\Zp \sim P} \left[ \xsloss{\rv{f}}(\Zp) \right] \right]
\stochleq^*_{n \eta}  \frac{1}{n} \cdot \fcompnew(\cF , \hat{\Pi}, w , Z^n).
\end{align}
\end{theorem}
The proof is in the appendix; it is merely a sequence of
straightforward rewritings, where the key observation is that for
every $f \in \cF$, the annealed risk
$\Expann{\eta}_{\Zp \sim P} \left[ \xsloss{{f}}(\Zp) \right]$ is
related to the normalization factor appearing in the definition
\eqref{eq:entropified} of the probability density $q_f$ and its
$n$-fold product $C(f)$ appearing in \eqref{eq:shtarkovdet} via the
following equality, as follows immediately from the definitions:
\begin{align}\label{eq:firstthm}
\Expann{\eta}_{\Zp \sim P} \left[ \xsloss{{f}}(\Zp) \right] = \frac{1}{n} \cdot 
\frac{- \log C(f)}{\eta} .
\end{align}
We note that, by taking $w$ as in Proposition~\ref{prop:newgeneration},
via \eqref{eq:infcomp}, this theorem strictly generalizes Theorem~3.1
of \cite{zhang2006information}, the left-hand side of Zhang's
inequality being equal to the annealed excess risk and the right-hand
side to the information complexity, i.e.~the right hand of
\eqref{eq:infcomp}. However, by taking different $w$, we get different
bounds which are not covered by Zhang's results and which, as we'll see, can be used to recover minimax excess risk bounds for certain large classes of polynomial entropy.

The above ESI's have annealed expectations on their left-hand sides and thus still fall short of providing such excess risk bounds. 

This gap can be resolved under the \emph{$v$-central condition}.

\begin{definition}[$\mathbf{v}$-central condition] \label{def:v-central}
Let $v \colon x \mapsto \eta_0 \cdot x^\alpha$ be a function with domain $[0, \infty)$ for some $\eta_0 > 0$ and $\alpha \geq 0$. 
We say that $(P, \loss, \F)$ satisfies the $v$-central condition if, for all $\gamma > 0$,
\begin{align*}
\E [ e^{-v(\gamma) \xsloss{f}(Z)} ] \leq e^{v(\gamma) \cdot \gamma} .
\end{align*}
In the special case of $\alpha = 0$, we say that the $\eta$-central condition holds (for $\eta = \eta_0$).
\end{definition}
If the loss is $\eta$-exp-concave and the class $\F$ is convex, it is known that the $\eta$-central condition holds (see the figure on page 1798 of \Citet{vanerven2015fast}, or Lemma 1 of \cite{mehta2017fast} for an explicit proof). More generally, in the case of bounded losses the $v$-central condition is in fact equivalent to the well-known \emph{Bernstein condition}.

\begin{definition}[Bernstein condition] \label{def:bernstein}
Let $\beta \in [0, 1]$. We say that $(P, \loss, \F)$ satisfies the $\beta$-Bernstein condition if, for some constant $\bernc < \infty$
\begin{align*}
\E \left[ \xsloss{f}(Z)^2 \right] \leq \bernc \E \left[ \xsloss{f}(Z) \right]^\beta \qquad \text{for all } f \in \F .
\end{align*}
\end{definition}

We only recall one direction of the equivalence of the $v$-central condition to the Bernstein condition here; the full equivalence is due to \Citep[Theorem 5.4]{vanerven2015fast}.
\begin{lemma}[Bernstein implies $v$-central] \label{lemma:bernstein-v-central}
Assume for all $f \in \F$ that $R_f(Z) \in [-1/2, 1/2]$ a.s. 
If the $\beta$-Bernstein condition holds for some $\beta \in [0, 1]$ and some constant $\bernc$, 
then the $v$-central condition holds for
\begin{align*}
v(\gamma) = \min \left\{ \frac{\gamma^{1 - \beta}}{\bernc}, 1 \right\} . 
\end{align*}
\end{lemma}
\begin{proof}
For clarity, let $\bar{a}$ and $\bar{b}$ refer to the constants $a$ and $b$ from part 1(a) of Theorem 5.4 of \Citet{vanerven2015fast}. Apply that result with $\bar{b} = \frac{1}{2 \bar{a}}$, $\bar{a} = 1/2$, and the $u$ function there set to $x \mapsto \bernc x^\beta$. Note that although the statement of Theorem 5.4 actually imposes the stronger condition that the loss $\loss$ be $[0,1/2]$-valued, the proof thereof only requires that $R_f \in [-1/2, 1/2]$ a.s.~for all $f \in \F$.
\end{proof}
Note that for such bounded loss functions, the weakest Bernstein condition with $\beta = 0$ holds automatically and so does the $v$-central condition with $v(\gamma) \propto \gamma$.

The following lemma is a translation of Lemma 2 of \cite{grunwald2012safe} which addresses the aforementioned gap between the annealed and actual expectations.
\begin{lemma} \label{lemma:kl-renyi}
Suppose that the $v$-central condition holds for some $\eta_0 > 0$ and $\alpha \geq 0$. 
If $R_f(Z) \in [-1/2, 1/2]$ a.s., then for all $\gamma > 0$, for all $\eta \leq \frac{v(\gamma)}{2}$
\begin{align*}
\E_{Z \sim P} \left[ \xsloss{f}(Z) \right] 
\leq C_\eta \cdot \Expann{\eta}_{Z \sim P} \left[ \xsloss{f}(Z) \right] + \frac{C_\eta - 1}{\eta} v(\gamma) \cdot \gamma ,
\end{align*}
with $C_\eta = 2 + 2 \eta$. 
In particular, taking $\eta = v(\gamma) / 2$, we have
\begin{align*}
\E_{Z \sim P} \left[ \xsloss{f}(Z) \right] 
\leq C_{v(\gamma)/2} \cdot \Expann{v(\gamma)/2}_{Z \sim P} \left[ \xsloss{f}(Z) \right] + 2 (C_{v(\gamma)/2} - 1) \gamma .
\end{align*}
\end{lemma}
We note that a version of the above result also holds for general bounded losses. In fact, \cite{grunwald2016fast} (still under review) provide a refined version of this lemma that works even for unbounded losses and for any $\eta < v(\gamma)$, with sharper bounds for $\eta < v(\gamma)/2$.

With all the pieces in place, the next two excess risk bounds in terms of $\compnew$ are nearly immediate.
\begin{corollary} \label{cor:risk-comp-esi}
Take the same setup as Lemma~\ref{lemma:kl-renyi}. 
If $\hat{\Pi}$ is an arbitrary randomized estimator, and $w$ is an arbitrary 
luckiness function, then
\begin{align}
\E_{\rv{f} \sim \hat{\Pi}_{| \Zn}} \left[ 
    \E_{Z \sim P} \left[ \xsloss{\rv{f}}(Z) \right] 
\right]
\stochleq_{v(\gamma) \cdot n / 6}
    \frac{3 \, \fcompnew_{v(\gamma)/2}(\cF , \hat{\Pi}, w , Z^n)}{n} 
    + 4 \gamma . \label{eqn:risk-comp-rand}
\end{align}
If $\hat{f}$ is ERM, then
\begin{align}
\E_{Z \sim P} \left[ \xsloss{\hat{f}}(Z) \right] 
\stochleq_{v(\gamma) \cdot n/6} \frac{3 \, \compnew_{v(\gamma)/2}(\cF , \hat{f})}{n} + 4 \gamma . \label{eqn:risk-comp-erm}
\end{align}
\end{corollary}
\begin{proof}
For \eqref{eqn:risk-comp-rand}, start with Lemma~\ref{lemma:kl-renyi} with $\eta = v(\gamma)/2$ for the desired $\gamma > 0$, and then apply Theorem~\ref{thm:first} to (stochastically) upper bound the annealed excess risk term. Observe that since $v(\gamma) \leq 1$ by assumption, we have $C_{v(\gamma)/2} \leq 3$.

For \eqref{eqn:risk-comp-erm}, start with \eqref{eqn:risk-comp-rand}, and take $w \equiv 1$  and $\hat{\Pi}(\cdot \mid \Z^n)$  equal to the probability measure that places mass $1$ on $\hat{f}_{\mid \Zn}$ and $0$ elsewhere. From these settings and the optimality of ERM for the empirical risk, $\fcompnew(\cF , \hat{\Pi}, w , Z^n)$ reduces to the simpler form $\compnew(\cF, \hat{f})$.
\end{proof}

To aid in the interpretation of the corollary, let us remark on two
special cases of \eqref{eqn:risk-comp-rand}. In both cases, we will
suppose that, as in Proposition~\ref{prop:newgeneration},
$w(z^n,f) := \pi(f) /\hat{\pi}(f \mid z^n)$ where $\pi$ is the density
of a fixed probability measureon $\cF$ independent of the sample, so
that $\compnew$ is bounded by information complexity.  First, if the
$\eta$-central condition holds, then, setting $\eta' = \eta/2$ and
using \eqref{eq:infcomp}, it further follows that
\begin{align*}
\E_{\rv{f} \sim \hat{\Pi}_{| \Zn}} \left[ 
    \E_{Z \sim P} \left[ \xsloss{\rv{f}}(Z) \right] 
\right] 
\stochleq_{\frac{n \eta}{6}} 
  \frac{3}{n} \left( 
        \E_{\rv{f} \sim \hat{\Pi} \mid \Zn}[R_{\rv{f}}(\Zn)] 
        + \frac{2 \KL(\; ( \hat{\Pi} \mid \Zn) \pipes \Pi \; )}{ \eta} 
    \right) .
\end{align*}
In the second case, we take $\hat{\Pi}$ to be any posterior whose
$\hat{\Pi}$-expected empirical risk is at most the empirical risk of
$f^*$ (for example, $\hat{\Pi}$ could be Dirac measure on, hence
coincide with, ERM), and for simplicity we further assume that a
$\beta$-Bernstein condition holds for some $\bernc \geq 2$ (if it
holds for a smaller $\bernc$, we will simply weaken the
condition). Thus, from the bounded loss assumption the $v$-central
condition holds for $v(\gamma) = \frac{\gamma^{1-\beta}}{\bernc}$
(provided that we only consider $\gamma \leq B^{1/(1- \beta)}$), and
tuning $\gamma$ yields 
$\gamma = A_1 \cdot n^{-1/(2-\beta)} \KL(\; ( \hat{\Pi} \mid \Zn) \pipes \Pi\; )^{1/(2-\beta)}$ 
for a constant $A_1$ depending only on $\beta$ and $B$, so that 
\begin{align}\label{eq:finiteexample}
    \E_{\rv{f} \sim \hat{\Pi}_{| \Zn}} \left[ \E_{Z \sim P} \left[
        \xsloss{\rv{f}}(Z) \right] \right] \stochleq_{n \cdot a_n} A_1 \cdot
    \left( \frac{ \KL(\; ( \hat{\Pi} \mid \Zn) \pipes \Pi\; )}{n}
    \right)^{1 / (2 - \beta)} ,
  \end{align}
where $a_n = A_2 ( \KL(\; ( \hat{\Pi} \mid \Zn) \pipes \Pi \; )/n)^{(1- \beta) / (2 - \beta)}$ 
for a constant $A_2$ depending only on $\beta$ and $B$. 
Lastly, as usual, in both cases when the class is finite and the prior $\Pi$ is uniform, 
the KL-divergence term reduces to $\log |\F|$. We thus retrieve the
familar bounds $O(n^{-1/2})$ in the worst-case ($\beta = 0$, for which
the Bernstein condition holds vacuously for bounded losses) and
$O(n^{-1})$ for the best case, $\beta = 1$.

In the next section, we derive excess risk bounds for large classes by suitably controlling $\compnew$ and then applying Corollary~\ref{cor:risk-comp-esi}.

\section{Bounds on Maximal Complexity $\compnew(\cF)$ and the 
excess risk bounds they imply}
\label{sec:complexity}

The results of the previous section do not yet yield explicit excess risk bounds as they still involve a $\compnew$ term. In this section, we leverage and extend ideas from \cite{opper1999worst} as well as results from empirical process theory to provide explicit bounds on $\compnew$ for several important types of large classes: classes of VC-type, classes whose empirical entropy grows polynomially, and sets of classifiers whose entropy with bracketing grows polynomially. Along the way, we form vital connections to expected suprema of certain empirical processes, including Rademacher complexity. 
At the end of this section we present explicit excess risk bounds; these bounds are simple consequences of the bounds we developed on $\compnew$.

\subsection{Preliminaries}
\label{sec:prelloss}
\paragraph{Losses and Lipschitzness.}
To properly capture losses like log-loss and supervised losses like 0-1 loss and squared loss, 
we introduce two different parameterizations of the loss function:
\begin{enumerate}
\item the \emph{direct} parameterization: $\loss_f(z) = f(x,y)$;
\item and the \emph{supervised loss} parameterization: $\loss_f(z) = \loss(y, f(x))$.
\end{enumerate}
For example, in the case of conditional density estimation with log loss, we then have $\loss_f(z) = f(x, y) = -\log p_f(y \mid x)$. Thus, each function $f \in \F$ has domain $\Z$, and the equivalence $\F = \{\loss_f : f \in \F\}$ holds. On the other hand, for supervised losses, each function $f \in \F$ has domain $\X$ while each loss-composed function $\loss_f$ has domain $\Z$.

Unlike previous sections, in this section we require an additional assumption in the case of the supervised loss parameterization: we assume that, for each outcome $(x,y) = z \in \Z$, 
the loss $\loss_f(z) = \loss(y, f(x))$ is $\lip$-Lipschitz in its second argument, i.e.~for all $f, g \in \F$,
\begin{align}
\left| \loss(y, f(x)) - \loss(y, g(x)) \right| \leq \lip \left| f(x) - g(x) \right| . 
\tag{A2} \label{eqn:lipschitz}
\end{align}
In the case of classification with 0-1 loss, $\F$ is the set of
classifiers taking values in $\{0, 1\}$ and $\Y = \{0, 1\}$, and so
\eqref{eqn:lipschitz} will hold with $L = 1$ (and is in fact an
equality).  For convenience in the analysis, in the case of the direct
parameterization we may always take $L = 1$. 

\paragraph{Standard complexity measures.}

It will be useful to review some of the standard notions of complexity before presenting our bounds. In the below, let $\cH$ be a class of functions mapping from some space $\cS$ to $\reals$; we typically will take $\cS$ equal to either $\X$ or $\Z$.

%% Covering numbers
For a pseudonorm $\|\cdot\|$, the \glsadd{coveringNumber} \emph{$\varepsilon$-covering number} $\N(\cH, \|\cdot\|, \varepsilon)$ is the minimum number of radius-$\varepsilon$ balls in the pseudonorm $\|\cdot\|$ whose union contains $\cH$. 
We will work with the $L_2(Q)$ (or $L_1(Q)$) pseudonorms for some probability measure $Q$. A case that will occur frequently is when $Q = P_n$ is the empirical measure $\frac{1}{n} \sum_{j=1}^n \delta_{S_j}$ based on a sample $S_1, \ldots, S_n$; here, $\delta_s$ (for $s \in \cS$) is a Dirac measure, and the sample will always be clear from the context.

%% Bracketing numbers
For two functions $h^{(l)}$ and $h^{(u)}$, the \emph{bracket} $[h^{(l)}, h^{(u)}]$ is the set of all functions $f$ that satisfy $h^{(l)} \leq f \leq h^{(u)}$. 
An \emph{$\varepsilon$-bracket} (in some pseudonorm $\|\cdot\|$) is a bracket $[h^{(l)}, h^{(u)}]$ satisfying $\|h^{(l)} - h^{(u)}\| \leq \varepsilon$. 
The \glsadd{bracketingNumber} \emph{$\varepsilon$-bracketing number} $\N_{[\cdot]}(\cH, \|\cdot\|, \varepsilon)$ is the minimum number of $\varepsilon$-brackets that cover $\cH$; the logarithm of the $\varepsilon$-bracketing number is called the \emph{$\varepsilon$-entropy with bracketing}.

%% Rademacher complexity
Let $\epsilon_1, \ldots, \epsilon_n$ be independent Rademacher random variables (distributed uniformly on $\{-1, 1\}$). 
The \glsadd{empRadComp} 
\emph{empirical Rademacher complexity} of $\cH$ and the \glsadd{radComp} \emph{Rademacher complexity} of $\cH$ respectively are
\begin{align*}
\rad_n(\cH \mid S_1, \ldots, S_n) 
:= \E_{\epsilon_1, \ldots, \epsilon_n} \Biggl[ 
        \sup_{h \in \cH} \left| \frac{1}{n} \sum_{i=1}^n \epsilon_i h(S_i) \right| 
    \Biggr] 
\qquad \qquad
\rad_n(\cH) 
:= \E \Biggl[ 
        \sup_{h \in \cH} \left| \frac{1}{n} \sum_{i=1}^n \epsilon_i h(S_i) \right| 
    \Biggr] ,
\end{align*}
where the first expectation is conditional on $S_1, \ldots, S_n$.

\subsection{$H$-local complexity and Rademacher complexity bounds on the NML complexity}
\label{sec:rademacher}
We first show that the simple form of the complexity 
$\compnew(\cF, \hat{f}) \leq \compnew(\cF)$ can be directly upper bounded in terms of two other complexity notions, 
the $H$-local complexity (defined below) and Rademacher complexity, 
up to a constant depending on the $L_2(P)$ diameter of $\cF$.
\begin{theorem}
\label{thm:opper-haussler-talagrand}
Fix $\varepsilon > 0$ and let $\cF$ have diameter $\varepsilon$ in the $L_2(P)$ pseudometric. Define   $\sigma := e \, \lip \, \varepsilon$, fix arbitrary $f_0 \in \cF$, and define the loss class $\G := \{ \loss_{f_0} - \loss_f : f \in \cF \}$. 
Define \glsadd{Tn}
\begin{align*}
T_n 
:= \sup_{f \in \cF} \left\{ 
    \sum_{j=1}^n \left( \loss_{f_0}(Z_j) - \loss_f(Z_j) \right)
    - \E_{\Zn \sim Q_{f_0}} \left[ \sum_{j=1}^n \left( \loss_{f_0}(Z_j) - \loss_f(Z_j) \right) \right] 
\right\}.
\end{align*}
Then
\begin{align}
\compnew_\eta(\cF) 
&\leq   3 \E_{\Zn \sim Q_{f_0}} \left[ T_n \right] + n \eta \sigma^2  \label{eqn:oht-1-first} \\
&\leq   6 n \E_{\Zn \sim Q_{f_0}} \left[ \rad_n(\G) \right] + n \eta \sigma^2  . \label{eqn:oht-2-first}
\end{align}
\end{theorem}

Two remarks are in order. First, we refer to the quantity $\E_{\Zn \sim Q_{f_0}} [ T_n ]$ as an \glsadd{HLocalComplexity} \emph{entropified local complexity}, or \emph{$H$-local complexity} for short. The ``local'' part of the name stems from how, in the empirical process inside the supremum defining $T_n$, the losses are centered/localized around $\loss_{f_0}$;  the ``entropified'' part of the name is due to the fact that the sample is distributed according to $Q_{f_0}$, itself defined via entropification. 
Second, the attentive reader may have noticed that in the above theorem, 
the expectation in the Rademacher complexity is relative to the distribution $Q_{f_0}$ for arbitrary $f_0 \in \cF$, rather than the distribution $P$ generating the data. Moreover, the appearance of $Q_{f_0}$ appears to dampen the utility of $\F$ having small $L_2(P)$ diameter. This apparent mismatch will be of no concern due to a technical lemma (Lemma~\ref{lemma:sigma} in Appendix~\ref{app:opper-haussler-talagrand}), which relates the $L_2(Q_{f_0})$ and $L_2(P)$ pseudometrics.

We now sketch a proof of this theorem in three steps; the first part, \eqref{eqn:oht-1-first}, is an immediate consequence of Lemmas~\ref{lemma:opper-haussler} and \ref{lemma:mostly-talagrand} below. 
The proofs of these results can be found in Appendix~\ref{app:opper-haussler-talagrand}. 

\paragraph{First step: Relating $\compnew$ to exponential moment of $T_n$.}

The following result follows from a straightforward generalization of an argument of \cite{opper1999worst}:
\begin{lemma} \label{lemma:opper-haussler}
Take arbitrary $\cF$ and fix arbitrary  $f_0 \in \cF$. Then:
\begin{align}
\compnew_\eta(\cF) 
\leq \frac{1}{\eta} \log  \E_{\Zn \sim Q_{f_0}} \left[ e^{\eta T_n } \right].
\label{eqn:cond-exp-moment-T}
\end{align}
\end{lemma}

\paragraph{Second step: Bounding exponential moment of $T_n$.}
It remains to bound $\E [ e^{\eta T_n}  ]$. The next lemma does this by leveraging Talagrand's inequality.

\begin{lemma} \label{lemma:mostly-talagrand}
Suppose that $\cF$ has $L_2(P)$ diameter at most $\epsilon$. 
Recall that $\sigma = e \, \lip \, \varepsilon$ and $\G := \{ \loss_{f_0} - \loss_f : f \in \cF \}$. 
Then
\begin{align}
\E_{Z^n \sim Q_{f_0}} [ e^{\eta T_n } ] 
\leq \exp \left( 3 \eta \E_{Z^n \sim Q_{f_0}} \left[ T_n \right] + n \eta^2 \sigma^2 \right) \label{eqn:talagrand-1} 
\end{align}
\end{lemma}
We note that \cite{opper1999worst} obtained a result similar in form to \eqref{eqn:talagrand-1} but under the considerably stronger assumption that the original class has finite $\sup$-norm entropy and, consequently, that the class $\G_k$ has $\sup$-norm radius at most $O(\varepsilon)$.

Inequality \eqref{eqn:oht-1-first} of Theorem~\ref{thm:opper-haussler-talagrand} now follows. Proving the second part, inequality \eqref{eqn:oht-2-first}, is based on the following standard result from empirical process theory.

\begin{lemma} \label{lemma:E-sup-rad}
\begin{align*}
\E_{Z^n \sim Q_{f_0}} \left[ T_n \right] 
\leq 2 n \E_{\Zn \sim Q_{f_0}} \left[ \rad_n(\G) \right] .
\end{align*}
\end{lemma}
This lemma is proved in Appendix~\ref{app:opper-haussler-talagrand} for completeness.

To make the bounds in Theorem~\ref{thm:opper-haussler-talagrand} useful for general $\cF$ with possibly large $L_2(P)$ diameter, we first decompose $\compnew_{\eta}(\cF)$ in
terms of the $L_2(P)$ covering numbers at some small, optimally-tuned resolution
$\varepsilon$ and the maximal complexity among all Voronoi cells induced by
the cover, as in \eqref{eq:maxcomplexity}.  We then use existing bounds on 
$H$-local complexity 
and Rademacher complexity to get sharp bounds on
$\compnew(\cF)$ in terms of covering numbers.

To this end, let $\cF$ be arbitrary and let $\{f_1, f_2, \ldots,
f_{N_\varepsilon}\}$ form an $(\varepsilon / 2)$-net for $\cF$ in the
$L_2(P)$ pseudometric, with $N_\varepsilon := \N(\F, L_2(P),
\varepsilon/2)$, and let $\Fpart{1}, \ldots, \Fpart{N_\varepsilon}$ be
the corresponding partition of $\F$ into Voronoi cells according to
the $L_2(P)$ pseudometric. That is, for each $k \in [N_\varepsilon]$,
the Voronoi cell $\Fpart{k}$ is defined as $\bigl\{ f \in \F : k =
\argmin_{i \in [N_\varepsilon]} \|f - f_i\|_{L_2(P)}
\bigr\}$.\footnote{Ties are broken arbitrarily.}  Clearly, each cell
$\Fpart{k}$ has $L_2(P)$ diameter at most $\varepsilon$. For each $k$,
fix an arbitrary $f_k \in \Fpart{k}$ and let $T_n^{(k)}$ be defined as
$T_n$ above with $f_k$ in the role of $f_0$. Inequality \eqref{eq:decomposenml} immediately gives the following corollary of Theorem~\ref{thm:opper-haussler-talagrand}.
\begin{corollary} \label{cor:opper-haussler-talagrand}
Let $\sigma := e \, \lip \, \varepsilon$ and, for each $k \in [N_\varepsilon]$, define the loss class $\G_k := \{ \loss_{f_k} - \loss_f : f \in \Fpart{k} \}$. 
Then
\begin{align}
\compnew_\eta(\cF) 
&\leq \frac{\log N_\varepsilon}{\eta}
       + \max_{k \in [N_\varepsilon]} \left( 3 \E_{\Zn \sim Q_{f_k}} \left[ T_n^{(k)} \right] + \eta n \sigma^2 \right) \label{eqn:oht-1} \\
&\leq \frac{\log N_\varepsilon}{\eta}
       + \max_{k \in [N_\varepsilon]} \left( 6 n \E_{\Zn \sim Q_{f_k}} \left[ \rad_n(\G_k) \right] + \eta n  \sigma^2 \right) . \label{eqn:oht-2}
\end{align}
\end{corollary}
\subsection{From $H$-local complexity and Rademacher complexity to excess risk bounds}
We  now show some concrete implications of our link between $\compnew$,
$\E [ T_n ]$ and $\rad_n$ for three types of classes: classes of VC-type,
classes with polynomial empirical entropy, and sets of classifiers
of polynomial $L_2(P)$ entropy with bracketing. Each of these types
of classes will be defined in sequence. Let $\cH$ be a class of
functions over a space $\cS$. The class $\cH$ is said to be of
\emph{VC-type} if, for some $A \in (0, \infty)$ and $V > 0$, for all
$\varepsilon > 0$, the empirical covering numbers of $\G$ satisfy
\begin{align}
\sup_{s_1, \ldots, s_n \in \cS} \N(\cH, L_2(P_n), \varepsilon) 
\leq \left( A / \varepsilon \right)^V . \label{eqn:vc-type-classes}
\end{align}
Such classes often are called parametric classes.

The class $\cH$ is said to have \emph{polynomial empirical entropy} if, for some $A  \in (0, \infty), \rho \in (0,1)$, for all $\varepsilon > 0$, the empirical entropy of $\cH$ satisfies
\begin{align}
\sup_{s_1, \ldots, s_n  \in \cS} \log \N(\cH, L_2(P_n), \varepsilon) 
\leq \left( A / \varepsilon \right)^{2 \rho} . \label{eqn:poly-classes}
\end{align}
These classes are nonparametric.

%% Bracketing numbers
We say the class $\cH$ has \emph{polynomial $L_1(P)$ entropy with bracketing} if, for some $A  \in (0, \infty), \rho \in (0,1)$, for all $\varepsilon > 0$, the $L_1(P)$ entropy with bracketing of $\G$ satisfies
\begin{align}
\log \N_{[\cdot]}(\cH, L_1(P), \varepsilon) 
\leq \left( A^2 / \varepsilon \right)^\rho . \label{eqn:poly-classes-b}
\end{align}

To obtain explicit bounds from Corollary~\ref{cor:opper-haussler-talagrand}, we require suitable upper bounds on either 
the Rademacher complexity $\E_{\Zn \sim Q_{f_k}} \left[ \rad_n(\G_k) \right]$ 
or directly on the $H$-local complexity $\E_{Q_{f_k}} \left[ T_n^{(k)} \right]$ itself 
in the three cases of interest: VC-type classes, classes of polynomial empirical entropy, and sets of classifiers of polynomial entropy with bracketing.

It is simple to obtain such bounds using Dudley's entropy integral, itself a product of the well-known chaining method from empirical process theory. However, the trick here is that we would like to leverage the fact that $\G_k$ has small $L_2(P)$ diameter. 
By making use of Talagrand's generic chaining complexity,  \cite{koltchinskii2011oracle} obtained bounds which improve with reductions in the $L_2(P)$ diameter. We restate simplified versions of these bounds here 
(see equations (3.17) and (3.19) of \cite{koltchinskii2011oracle}):

In the following and all subsequent results, 
$\lesssim$ indicates inequality up to multiplication by a universal
constant.
\begin{theorem}[Rademacher complexity bounds (Koltchinskii 2011)] \label{thm:small-rad}
Let $\cH$ be a class of functions over $\Z$, and let $Q \in \Delta(\Z)$. 
Let  $\sup_{h \in \cH} \E_{Z \sim Q} [ h(Z)^2 ] \leq \sigma^2$ and $U := \sup_{h \in \cH} \|h\|_\infty$. 
Assume that $\cH$ is of VC-type as in \eqref{eqn:vc-type-classes} with exponent $V$. 
Then, for $\sigma^2 \geq \frac{c}{n}$ (for some constant $c$)
\begin{align}
\E_{\Zn \sim Q} \left[ \rad_n(\cH) \right] 
\lesssim \max \left\{ 
                   \sqrt{\frac{V}{n}} \sigma \sqrt{\log \frac{A}{\sigma}},
                            \frac{V U}{n} \log \frac{A}{\sigma} 
               \right\} . \label{eqn:vc-rad-bound}
\end{align}

Take $\cH$, $Q$, $\sigma$, and $U$ as before, but now assume that 
$\G$ is of polynomial empirical entropy as in \eqref{eqn:poly-classes} with exponent $\rho$. 
Then
\begin{align}
\E_{\Zn \sim Q} \left[ \rad_n(\cH) \right] 
\lesssim \max \left\{ 
                   \frac{A^\rho}{\sqrt{n}} \sigma^{1 - \rho} , 
                   \frac{A^{2 \rho / (\rho + 1)} U^{(1 - \rho) / (1 + \rho)}}{n^{1/(1+\rho)}} 
               \right\} . \label{eqn:poly-rad-bound}
\end{align}
\end{theorem}
For the case of classes of polynomial entropy with bracketing, we
appeal to upper bounds on $\E_{Q_{f_k}} \left[ T_n^{(k)} \right]$.  If
the class $\G_k$ has small $L_1(Q_{f_k})$ diameter and, moreover, if
it also has polynomial $L_1(Q_{f_k})$ entropy with bracketing, then
Lemma A.4 of \cite{massart2006risk} provides precisely such a
bound. Below, we present a straightforward consequence thereof. 
\begin{theorem}[Expected supremum bounds (Massart and N\'ed\'elec 2006)] \label{thm:small-Esup}
Let $\cH$ be a class of functions over $\Z$, and let $Q \in \Delta(\Z)$. 
Let  $\sup_{h \in \cH} \E_{Z \sim Q} [ |h(Z)| ] \leq \sigma^2$ and $\sup_{h \in \cH} \|h\|_\infty \leq 1$. 
Assume that $\cH$ is has polynomial entropy with bracketing as in \eqref{eqn:poly-classes-b} with exponent $\rho$. 
Then
\begin{align}
\E_{\Zn \sim Q} \left[ \sup_{h \in \cH} \left\{ \frac{1}{n} \sum_{j=1}^n h(Z_j) - \E [ h(Z) ] \right\} \right] 
\lesssim \max \left\{ 
                   \frac{A^\rho}{\sqrt{n}} \sigma^{1 - \rho} , 
                   \frac{A^{2 \rho / (\rho + 1)}}{n^{1/(1+\rho)}} 
               \right\} . \label{eqn:poly-E-sup-bound-b}
\end{align}
\end{theorem}

The following theorem builds on
Corollary~\ref{cor:opper-haussler-talagrand} and \emph{nearly} follows
by plugging in either \eqref{eqn:vc-rad-bound} or
\eqref{eqn:poly-rad-bound} into \eqref{eqn:oht-2} and tuning
$\varepsilon$ in terms of $n$ and $\eta$ (which gives the VC case,
\eqref{eqn:comp-vc}) , and plugging in \eqref{eqn:poly-E-sup-bound-b}
into \eqref{eqn:oht-1} and then tuning (which gives the polynomial
entropy case, \eqref{eqn:comp-poly-newbee}). The remaining work is to
resolve a minor discrepancy between $L_2(P)$ pseudonorms and
$L_2(Q_{f_k})$ pseudonorms (or the $L_1$ versions thereof). This
theorem will allow us to show optimal rates under Bernstein
conditions.

\begin{theorem} \label{thm:comp-bound}
  If $\F$ is of VC-type as in \eqref{eqn:vc-type-classes} with
  exponent $V$, then for all $\eta \in (0, 1]$, 
\begin{align}
\frac{\compnew_\eta(\F)}{n} 
\lesssim V \log \frac{A \lip n}{V} \; \cdot \; n^{-1} \cdot \eta^{-1} . \label{eqn:comp-vc}
\end{align}
If $\F$ has polynomial empirical entropy as in \eqref{eqn:poly-classes} or is a set of classifiers of polynomial entropy with bracketing as in \eqref{eqn:poly-classes-b} with exponent $\rho$, 
{then}, for all $0 < \eta < 1$, 
\begin{align}
\label{eqn:comp-poly-newbee}
\frac{\compnew_\eta(\F)}{n} \lesssim (A \lip)^{\frac{2 \rho}{1+\rho}}
\cdot n^{-\frac{1}{1 + \rho}} \cdot \eta^{- \frac{1-\rho}{1+\rho}}.
\end{align}
\end{theorem}
The proof of Theorem~\ref{thm:comp-bound} can be found in Appendix~\ref{app:comp-bound}. 
We will now prepare for our results on the rates of ERM on a
class $\cF$. We considerably generalize these results in the
next section, using the concept of `ERM-like'
estimators. For now, the reader may skip the following definition and
simply equate ERM-like with `ERM'.
\begin{definition}\label{def:ermlike}
Consider two models $\bar{\cF}$ and $\cF \subseteq \bar{\cF}$ and 
let $\hat{f}$ be
  any deterministic estimator on the larger
model $\bar{\cF}$ and $w: \cZ^n \rightarrow \reals^+_0$ be a luckiness function. 
We say that $\hat{f}$ is \emph{ERM}-like relative to $\cF$ and $w$ if for some $\tau \geq 0$, for all $\eta > 0$, all $z^n \in \cZ^n$, 
  \begin{equation}\label{eq:ermlike}
\fcompnew_{\eta}(\bar{\cF},\hat{f},w,z^n) \lesssim 
\compnew_{\eta}(\cF) + \frac{\tau}{\eta}.
\end{equation}
Note that if we take $\bar{\F} = \F$ and set $\hat{f}$ to be ERM on
$\cF$, then $\hat{f}$ is indeed ERM-like with $\tau= 0$, since then
$R_{\hat{f}}(z^n)\leq 0$.
\end{definition}
In the following corollary, note that in both cases, the occurrence of the  
Bernstein exponent $\beta$ (or $\kappa^{-1}$) is consistent with its 
occurence in the simple finite $\cF$ setting of \eqref{eq:finiteexample}.
\begin{corollary}\label{cor:ry}
  Assume that a $\beta$-Bernstein condition holds for $\cF$ as in
  Definition~\ref{def:bernstein} for some $\beta$ and $\bernc$, and impose assumption \eqref{eqn:bounded}. 
  Define $\kappa := \beta^{-1}$. Let $\hat{f}$ be a deterministic estimator on a model
  $\bar{\cF} \supseteq \cF$ that is ERM-like relative to $\cF$ and
  some luckiness function $w$.

Then, \eqref{eqn:comp-vc} further implies,
taking $\gamma = \left( \bernc \left( \frac{V \log \frac{A \lip n}{V}}{n}  + \tau \right) \right)^{\kappa / (2 \kappa - 1)}$ 
and taking $n$ large enough so that $\frac{V}{n} \log \frac{A \lip n}{V} + \tau \leq \bernc^{1 / (1 - \beta)}$,  that for all $z^n \in \cZ^n$, 
\begin{align}
\frac{\fcompnew_{v(\gamma)}(\bar{\F},\hat{f},W,z^n)}{n} + \gamma & 
\lesssim 
\left(
\bernc \left(V \log \frac{A \lip n}{V} + \tau \right) \right)^{\frac{\kappa}{2 \kappa - 1}} 
\cdot n^{- \frac{\kappa}{2 \kappa - 1}} . \label{eqn:kl-comp-bound-vc}
\end{align}
Analogously, under such a Bernstein condition,
\eqref{eqn:comp-poly-newbee} further implies, 
taking $\gamma = n^{-\frac{\kappa}{2 \kappa - 1 + \rho}}$ 
and assuming that $n > (2 \bernc)^{-\frac{2 \kappa -1 + \rho}{\kappa-1}}$, 
that for all $z^n \in \cZ^n$, 
{ 
\begin{align}
\frac{\fcompnew_{v(\gamma)}(\bar\F, \hat{f}, W, z^n)}{n} + \gamma 
\lesssim  \left( ( A L) ^{\frac{2 \rho}{\rho + 1}} +1 +\tau \right) \cdot
\bernc^{\frac{1 - \rho}{1 + \rho}}
n^{-\frac{\kappa}{2 \kappa - 1 + \rho}} . \label{eqn:kl-comp-bound-poly}
\end{align} }
\end{corollary}
We used the notation $\kappa = \beta^{-1}$ here to make the results easier comparable to \cite{tsybakov2004optimal} and \cite{audibert2004PAC}; the result still holds for the case $\beta = 0$ though, if we simply replace $\kappa^{-1}$ by $\beta$ in all exponents above; e.g.~$\kappa / (2 \kappa - 1)$ in \eqref{eqn:kl-comp-bound-vc} becomes $1/(2 - \kappa^{-1}) = 1/(2- \beta)$. 
\begin{proof}
  To see \eqref{eqn:kl-comp-bound-vc}, we begin by upper bounding
  $\frac{\compnew_{v(\gamma)}(\F)}{n}$ using
  \eqref{eqn:comp-vc} with $\eta = v(\gamma) = \min \left\{ \frac{\gamma^{1 - \beta}}{\bernc}, 1 \right\}$ (from Lemma~\ref{lemma:bernstein-v-central}). Tentatively suppose that $\bernc \gamma^{-(1-\beta)} \geq 1$; then 
$v(\gamma)^{-1} \lesssim \bernc \gamma^{-(1 - \beta)}$, and hence
\begin{align*}
\frac{\fcompnew_{v(\gamma)}(\bar{\F},\hat{f},W,z^n)}{n} + \gamma 
\lesssim \frac{\bernc}{n} \left( V \log \frac{A L n}{V} + \tau \right) \gamma^{-(1 - \beta)} + \gamma .
\end{align*}
Tuning $\gamma$ such that it is equal to the first term on the RHS above yields \eqref{eqn:kl-comp-bound-vc}; it is simple to verify that the supposition $\bernc \gamma^{-(1-\beta)} \geq 1$ is ensured by the constraint on $n$ stated in the corollary.

  We now prove \eqref{eqn:kl-comp-bound-poly}. We first prove the case
  for $\hat{f}$ ERM on $\cF$, i.e.~with $\compnew(\cF)$ rather than
  $\fcompnew(\bar{\cF})$ on the left. 
Let $\gamma_n := n^{-\frac{\kappa}{2 \kappa - 1 + \rho}}$ be the value of $\gamma$ used at sample size $n$. We get from the definition of the Bernstein
  condition and Lemma~\ref{lemma:bernstein-v-central} that
\[ \eta_n := v(\gamma_n) = \bernc^{-1} \left( n^{- \frac{\kappa}{2 \kappa - 1 +
    \rho}}\right)^{(\kappa-1)/\kappa}
= \bernc^{-1} n^{-\frac{\kappa-1}{2 \kappa - 1 +
    \rho}}
 \] for all $n$ for which the RHS above is at most $1$. This will hold whenever $n \geq (1 / \bernc)^{\frac{2 \kappa -1 + \rho}{\kappa-1}}$. 
For such $n$, we will also have $\eta_n \leq 1$ and thus can apply \eqref{eqn:comp-poly-newbee}, plugging in $\eta = \eta_n = v(\gamma_n)$. 
The result follows by simple algebra for all $n$ larger than the given bound. 
To extend the result to ERM-like estimators, we note that from 
\eqref{eq:ermlike} and
\eqref{eqn:comp-poly-newbee}, using $\rho \geq 0$, we get
\begin{align*}
\frac{\fcompnew_\eta(\cF,\hat{f},W,z^n)}{n} \lesssim \left( 
(A \lip)^{\frac{2 \rho}{1+\rho}}
\cdot n^{-\frac{1}{1 + \rho}} + \tau \right) \cdot \eta^{- \frac{1-\rho}{1+\rho}}.
\end{align*}
We now proceed as above plugging in $\eta_n$ into the above expression
rather than \eqref{eqn:comp-poly-newbee}.
\end{proof}

\subsection{Recovering Bounds under Bernstein Conditions for Large Classes}

Let us now see how Lemma~\ref{lemma:kl-renyi} can recover the known
optimal rates under the Tsybakov margin condition and also the best
known rates for empirical risk minimization under Bernstein-type
conditions. The theorem below works for general deterministic estimators, not just ERM; the (simple) rate implications for ERM are discussed underneath the theorem. 
 While there are other techniques that can achieve the same 
rates for ERM, we feel that our approach embodies a simpler analysis for the polynomial entropy case;
it also leads to some new results for other estimators, which we
detail in the next section.
\begin{theorem} \label{thm:excess-risk-erm} Assume that the
  $\beta$-Bernstein condition holds for $\cF$ as in Corollary~\ref{cor:ry} and
  define $\kappa := \beta^{-1}$. 
Let $\hat{f}$ be a deterministic estimator on $\bar{\cF} \supseteq \cF$ that is ERM-like relative to $\cF$ (with $\tau$ as in Definition~\ref{def:ermlike}) and some luckiness function $w$.
 First suppose  $\F$ is of VC-type as in
  \eqref{eqn:vc-type-classes} with exponent $V$. 
Then there is a universal constant $C_1$ such that for all $n$ large enough so that
$\frac{V}{n} \log \frac{A \lip n}{V} + \tau \leq \bernc^{1 / (1 - \beta)}$, 
we have
\begin{align}\label{eq:vcfinal}
  \E_{Z \sim P} \left[ \xsloss{\hat{f}}(Z) \right]
  \stochleq_{\psi_1(n)} C_1 \left(\bernc \left( V \log \frac{A \lip
        n}{V} + \tau \right) \right)^{\frac{\kappa}{2 \kappa - 1}} \cdot
  n^{-\frac{\kappa}{2 \kappa - 1}}, \end{align} where $\psi_1(n) =
\frac{1}{6 \bernc} \cdot \left( \bernc \left( V \log \frac{A L n}{V} +
    \tau \right) \right)^{(\kappa - 1) / (2 \kappa - 1)} \, n^{\kappa /
  (2 \kappa - 1)} \asymp (\log n)^{(\kappa - 1) / (2 \kappa - 1)}  \cdot n^{\kappa / (2 \kappa - 1)}$. Analogously,
suppose that $\F \subseteq \bar{\cF}$ has polynomial empirical entropy
as in \eqref{eqn:poly-classes} or is a set of classifiers of
polynomial entropy with bracketing as in \eqref{eqn:poly-classes-b}
with exponent $\rho$.  Then there is a $C_2$ such that for all $n$
large enough so that $n > (2 \bernc)^{-\frac{2 \kappa -1 +
    \rho}{\kappa-1}}$, we have 
\begin{align}\label{eq:polyfinal}
\E_{Z \sim P} \left[ \xsloss{\hat{f}}(Z) \right]
\stochleq_{\psi_2(n)} 
C_2 \left[ \left( (AL) ^{\frac{2 \rho}{\rho + 1}} + 1 + \tau \right)
\cdot \bernc^{\frac{1-\rho}{1+\rho}} \cdot n^{-\frac{\kappa}{2 \kappa - 1 + \rho}} \right],
\end{align}
where  
$\psi_2(n) = 
\frac{1}{6} 
\cdot n^{\frac{\kappa + \rho}{2 \kappa - 1 + \rho}}$.
\end{theorem}

\begin{Proof}[Proof of Theorem~\ref{thm:excess-risk-erm}]
For both results, observe from Corollary~\ref{cor:risk-comp-esi} and the definition of $\fcompnew$ that, for all $\gamma > 0$,
\begin{align}
\E_{Z \sim P} \left[ \xsloss{\hat{f}}(Z) \right]  \stochleq_{v(\gamma) \cdot n/6} \frac{3 \, 
\left(\fcompnew_{v(\gamma)/2}(\cF , \hat{f}) \right)}{n} + 4 \gamma. \label{eqn:risk-comp-erm-again-for-reader}
\end{align}
The result now follows by plugging in \eqref{eqn:kl-comp-bound-vc} and 
\eqref{eqn:kl-comp-bound-poly}. \end{Proof}

Theorem~\ref{thm:excess-risk-erm} combined with part (ii) of Proposition~\ref{prop:drop} implies that, with probability at least $1 - \delta$, ERM obtains the rate
\begin{align*}
n^{-\frac{\kappa}{2 \kappa - 1 + \rho}} 
+ n^{-\frac{\kappa + \rho}{2 \kappa - 1 + \rho}} \cdot \textstyle \log \frac{1}{\delta} .
\end{align*}
For sets of classifiers of polynomial entropy with bracketing,
the rate $n^{-\kappa / (2 \kappa - 1 + \rho)}$ is known to be optimal
and, in particular, matches the results of \cite{tsybakov2004optimal}
(see Theorem 1), \cite{audibert2004PAC} (see the discussion after
Theorem 3.3), and \citet[p.~36]{koltchinskii2006local}. Outside the
realm of classification, for classes with polynomial empirical
entropy, the rate we obtain is to our knowledge the best known for
ERM. In particular, if these nonparametric classes are convex and the loss is exp-concave, then $\kappa = \beta = 1$, and the rates we obtain for ERM are known to be minimax optimal \cite[Theorem 7]{rakhlin2017empirical}. 
We note, however, that there are cases where
$\beta < 1$ and yet, by using an aggregation scheme, one can obtain a
rate as if $\beta = 1$; one such example is in the case of squared
loss with a non-convex class
\citep{rakhlin2017empirical,liang2015learning}.

\section{Properties and Applications of $\compnew$}
\label{sec:applications}
Here we provide two applications of the developments in this paper,
emphasizing that we do not just recover existing results but also
generate new ones.  We first provide some general
properties of $\compnew(\cF, \hat{\Pi},w,\zn)$ that will be used in
the applications.

The first property concerns data-independent complexities, i.e.~based on constant luckiness functions.  Then for any randomized estimator, just as in the deterministic case, $\compnew(\cF, \hat{\Pi}, w, \zn)$ 
is upper bounded by the maximal complexity $\compnew(\cF)$:
\begin{proposition}\label{prop:maxcomp}
Let $\hat{\Pi}$ be an arbitrary randomized estimator 
and let $w(f,z^n) = 1$ for all $f \in\cF, z^n \in \cZ^n$. Then:  
\begin{equation}\label{eq:maxcomprandomized}
\sup_{z^n \in \cZ^n} \compnew(\cF, \hat{\Pi}, w, z^n) \leq \compnew(\cF).
\end{equation}
\end{proposition}
The second property is about a decomposition of $\compnew$ that vastly generalizes \eqref{eq:decomposenml}. 
Let $\cK$ be a countable set and consider a partition
$\{ \cF_k \mid k \in \cK\}$ of the model $\cF$.  Suppose we have an
estimator $\hat{\Pi}$ for $\cF$. This induces conditional estimators
$\hat{\Pi}_{|k}$ for each $\cF_k$ in the following way: for each $k$ and $z^n \in \cZ^n$, 
$\hat{\Pi}_{|k,z^n}$ is a distribution on $\cF_k$; if 
$\hat{\Pi}(f \in \cF_k \mid z^n) > 0$, then $\hat{\Pi}_{|k,z^n}$ is  the distribution of $f$ according to $\hat{\Pi}$ conditioned on
both data $z^n$ and $f \in \cF_k$; for $z^n$ with  
$\hat{\Pi}(f \in \cF_k \mid z^n) = 0$, $\hat{\Pi}_{|k,z^n}$ can be set to an arbitrary distribution on $\cF_k$ (the choice will not affect the results). Thus, $\hat{\Pi}_{|k}$ is now well-defined as an estimator that maps $\cZ^n$ into the set of distributions on $\cF_k$. For all $k \in \cK$,  
let $w_k$ be a  luckiness function on $\cZ^N \times \cF_k$. 
The following proposition shows
that, for arbitrary such luckiness functions $w_k$
and sub-models $\cF_k$, we can construct an overall luckiness function $w$
such that the complexity $\compnew(\cF, \hat{\Pi}, w, z^n)$ may be
decomposed into the sub-complexities
{$\compnew(\cF_k , \hat{\Pi}_{|k}, w_k , z^n)$. }

Formally, for each $z^n \in \cZ^n$, let $\hat{\Pi}_{\cK} \mid z^n$ be the
marginal distribution on $\cK$ with probability mass function 
$\hat{\pi}_{\cK} \mid z^n$ , induced by $\hat{\Pi}$, i.e.~$\hat{\pi}_{\cK}(k \mid z^n) 
= \hat{\Pi}(f \in \cF_k \mid z^n)$. Let $\Pi_{\cK}$ be a {prior} probability measure on $\cK$ with probability mass function  $\pi_{\cK}$.
\begin{proposition}\label{prop:decomposingcomp}
With the definitions above, for $k \in\cK$,  $f \in \cF_k$, $z^n \in \cZ^n$, set the global luckiness function $w$ to:
\begin{align}\label{eq:compositew}
w(z^n,f) := 
{w_{{k}}(z^n,{f})} \cdot \frac{{\pi}_{\cK}(k)}{\hat{\pi}_{\cK}({k} \mid z^n) }.
\end{align}
Then for each $z^n \in \cZ^n$: 
\begin{align} \label{eqn:decomposing-comp}
  \compnew(\cF, \hat{\Pi}, w, z^n)  
\leq \frac{\KL( \; ( \hat{\Pi}_{\cK} \mid  z^n) \pipes \Pi_{\cK})}{\eta}
        + \E_{\rv{k} \sim \hat{\Pi}_{\cK} \mid z^n } \left[
               \compnew(\cF_{\rv{k}} , \hat{\Pi}_{\mid \rv{k}} , w_{\rv{k}} , z^n)  
            \right] .
\end{align}
\end{proposition}
We note that this result generalizes two of our previous
observations. First, it generalizes \eqref{eq:infcomp} for the case of
countable $\cF$: the second term of \eqref{eq:infcomp} can be
retrieved by taking all the $\cF_k$ to be singletons and
$w_k \equiv 1$; note that the second term of
\eqref{eqn:decomposing-comp} then vanishes. Second, it generalizes
\eqref{eq:decomposenml}, which can be retrieved by taking $\cK$
finite, taking $\hat{\Pi}$ to be an arbitrary deterministic estimator,
setting the $w_k$ to be equal to $1$ and then applying
Proposition~\ref{prop:maxcomp}.

\subsection{Two-Part MDL Estimator Achieving Optimal Rate}
\label{sec:twopart} Consider a given (finite or countable) partition
$\{ \cF_k : k \in \cK \}$ of our model $\cF$ and let $\hat{f}_k$
represent ERM within $\cF_k$. Let, for each $k$, 
$\overline{\compnew}(\cF_k, \hat{f}_k)$ be any number larger than or equal
to $\compnew(\cF_k, \hat{f}_k)$.

Fix a `prior' distribution $\pi_{\cK}$ on $\cK$. 
Suppose that there exists a $k^* \in \cK$ that achieves excess risk  
$\inf_{k \in \cK} \inf_{f \in \cF_k} \E[ \xsloss{f}(Z)]$; 
if there's more than one $\cF_k$ achieving the minimum, 
take $k^*$ to be the one with the largest mass $\pi_{\cK}(k)$; 
further ties can be resolved arbitrarily. 
$f^*$ denotes the risk minimizer within $\cF_{k^*}$ which, consistently with
earlier notation, is then also the risk minimizer within $\cF$.

Consider a two-stage deterministic estimator $\ddot{f}$ that proceeds
by first selecting $\ddot{k} \equiv \ddot{k}_{|Z^n} \in \cK$ based on data $Z^n$ and then uses {the ERM} 
$\hat{f}_{\ddot{k}}$ within $\cF_{\ddot{k}}$. From
Proposition~\ref{prop:decomposingcomp}, we
see that, by choosing the luckiness functions $w$ and $w_k$
appropriately (all $w_k$ are set to $1$ for all $f \in \cF_k, z^n \in \cZ^n$), we get, with $\ddot{f}_{|Z^n} := \hat{f}_{\ddot{k}|Z^n}$ denoting the
$f \in \cF_{\ddot{k}_{|Z^n}}$ selected by the estimator $\ddot{f}$ upon observing
$Z^n$,
that
\begin{align}\label{eq:twoparty}
\fcompnew_{\eta}(\cF , \ddot{f}, w, z^n) \leq & - \sum_{i=1}^n \loss_{f^*}(z_i) +
 \nonumber \\ & \sum_{i=1}^n \loss_{\hat{f}_{\ddot{k} \mid z^n}}(z_i) 
+
\frac{- \log \pi_{\cK}(\ddot{k}_{|z^n} )}{\eta} 
+ \overline{\compnew}_{\eta}(\cF_{\ddot{k}}, \hat{f}_{\ddot{k}}).  
\end{align}
For reasons to become clear, we may call the particular choice of
estimator $\ddot{k}$ that is defined as minimizing the right-hand side
of the second line of \eqref{eq:twoparty} the \emph{$\eta$-generalized
  MDL estimator}. For this estimator we must then  further
have, for all $\eta > 0$, using that $\sum_{i=1}^n \loss_{\hat{f}_{\ddot{k} \mid z^n}}(z_i) 
+
\frac{- \log \pi_{\cK}(\ddot{k}_{|z^n} )}{\eta}  \leq \sum_{i=1}^n \loss_{\hat{f}_{k^* \mid z^n}}(z_i) 
+
\frac{- \log \pi_{\cK}(k^*)}{\eta} $, that
\begin{equation}\label{eq:later}
\fcompnew_{\eta}(\cF , \ddot{f}, W, z^n) \leq 
\frac{- \log \pi_{\cK}({k}^*)}{ \eta} + \overline{\compnew}_{\eta}(\cF_{k^*}, \hat{f}_{k^*}).
\end{equation}
Let us first consider the case that
$\overline{\compnew}_{\eta}(\cF_{k^*}, \hat{f}_{k^*}) =
{\compnew}_{\eta}(\cF_{k^*}, \hat{f}_{k^*})$. In that case, for each
$\eta > 0$, the $\eta$-generalized MDL estimator for $\cF$ has
essentially the same complexity bound as does ERM within the optimal
submodel $\cF_{k^*}$; indeed it is ERM-like according to
Definition~\ref{def:ermlike}. Thus, by
Theorem~\ref{thm:excess-risk-erm}, if a $\beta$-Bernstein condition
holds for $\cF_{k^*}$, the two-part MDL estimator achieves the same
rate as ERM within the optimal subclass $\cF_{k^*}$ if we choose
$\eta = v(\gamma)$ with $\gamma$ set to the same value as we would for
ERM relative to the submodel $\cF_{k^*}$.  Two-part $v(\gamma)$-MDL
thus serves as an optimal model selection criterion, and this holds
even if the number of alternatives $\cF_k$ considered is infinite.
However, even setting aside computational issues, there are two
obstacles to applying such an MDL principle in practice: first, for
each fixed $\eta$, ${\compnew}_{\eta}(\cF_{k^*}, \hat{f}_{k^*})$
depends on the unknown distribution $P$, and second, the desired
$\eta$, i.e, $\eta= v(\gamma)$ cannot be calculated since it depends
on the unknown $\beta$ in the Bernstein condition (and hence also on
$P$).

The first obstacle is overcome if, for each fixed $\eta$, we base the
two-part estimate $\ddot{f}$ on upper bounds of
$\overline{\compnew}_{\eta}(\cF_k, \hat{f}_{k})$ that can be
calculated for each $k$ without knowing $P$. If, for example, in the
polynomial entropy case, for each $k$ we plug in the upper bound
\eqref{eqn:comp-poly-newbee} with $\cF$ set to $\cF_k$ (which can in
principle be calculated), then 2-part MDL will still achieve the
optimal rate once we use the right $\eta$; similarly if $\cF$ is a
VC-class and we plug in, for each $k$, \eqref{eqn:comp-vc}.

As for the second obstacle, the optimal values of $\eta$ are induced by
the best $\beta$ for which a $\beta$-{Bernstein condition} holds;
{in practice, one can learn it from the data using an algorithm
  such as the `safe Bayes' algorithm of \cite{grunwald2012safe}.}

Remarkably, this model selection estimator has, for each fixed $\eta$,
an interpretation as minimizing a 2-part codelength of the data: in
the first part, one encodes a model index $k$ (using the code with lengths $- \log \pi_{\cK}(k)$; each prior induces such a code by Kraft's inequality) and in the second part,
one encodes the data using the NML code, i.e.~the optimal universal
code relative to $\cF_k$, and one picks the $k$ minimizing the
total codelength.  In fact, exactly this minimization, for the
case of $\eta = 1$ and log-loss, was suggested by
\cite{rissanen1996fisher} in the context of his MDL Principle, and has
been much applied since under the name {`refined MDL'}
\citep{grunwald2007the}. Rissanen suggested this method simply
because, viewed as a coding strategy, it led to small codelengths
(cumulative log loss) of the data, and gave no frequentist
justification in terms of convergence rates; we have just shown that,
with a correctly set $\eta$, optimal rates for ERM within the optimal
subclass can be recovered.

For the log-loss case with $\eta=1$, we get $\compnew(\cF_k,
\hat{f}_k) = \log \shtark(\cF_k, \hat{f}_k)$ and
\eqref{eq:shtarkovdet}, so the refined MDL estimator will pick the $k$
minimizing {\begin{equation}\label{eq:simplenml} \sum_{i=1}^n
    \loss_{\ddot{f}_{\ddot{k}|Z^n}}(Z_i) - \log \pi_{\cK}(k) + \log
    \int_{\Z^n} \hat{f}_{ \mid \zn}(\zn) d \nu(\zn),
  \end{equation} }
and thus avoids the problem of $\compnew$ being uncomputable without
knowledge of $P$.

\subsection{Penalized ERM Bounds: Lasso, Ridge and Luckiness NML}
\label{sec:penalized}
Consider a  \emph{penalization function} $\Gamma: \cF \rightarrow \reals$. Let $\hat{f}$ be the penalized empirical risk minimizer defined as 
\begin{equation}\label{eq:penalized}
\hat{f}_{|z^n} := \arg \min_{f \in \cF} \sum_{i=1}^n  \ell_f(z_i)
+ \eta^{-1} \cdot \Gamma(f) = \arg \min_{f \in \cF}
R_f(z^n) + \eta^{-1} \cdot \Gamma(f),
\end{equation}
with ties resolved arbitrarily; we assume that the minimum is always achieved. Obviously, successful estimation
procedures such as the lasso (with multiplier $\lambda := \eta^{-1}$)
and ridge regression \citep{HastieTF01} can be expressed in this
form. In this subsection we show how our results give  tight
annealed excess risk bounds for such estimators for arbitrary penalization functions
$\Gamma$; these can then be turned into real excess risk bounds using Corollary~\ref{cor:risk-comp-esi}. The main interest of this fact is that neither of the two already
pre-existing specializations of $\compnew$, i.e.~PAC-Bayesian
information complexity and Rademacher-type complexity, can easily
handle penalized ERM-type methods.  In contrast, we can simply define
the luckiness function $w(z^n) = \exp(-\Gamma(\hat{f}_{|z^n}))$. We can
then write $\fcompnew_{\eta}$ as
\begin{multline}
\fcompnew_{\eta}(\cF, \hat{f}, W, z^n) = 
R_{\hat{f}_{|z^n}}(z^n) + 
\frac{1}{\eta} \cdot \left(  \Gamma(\hat{f}_{|z^n})
+ \log \shtark_{\eta}(\cF, \hat{f}, w) \right)  \\
= \min_{f \in \cF} \left( \; R_f(z^n) + \frac{\Gamma(f)}{\eta} \; \right) 
+ \frac{1}{\eta} \log \shtark_{\eta}(\cF, \hat{f}, w).\label{eq:penalizedcompnew}
\end{multline}
with 
\begin{align*}
\shtark_{\eta}(\cF, \hat{f}, w) =
\E_{Z^n \sim P} \left[
\frac{e^{-
( \eta \xsloss{{\hat{f}_{|Z^n}}}(\Zn) + \Gamma(\hat{f}_{|Z^n}))}}{C(\hat{f}_{|Z^n})}
\right] = \int q_{\hat{f}_{|z^n}}(z^n) w(\hat{f}_{|z^n}) d\nu(z^n),
\end{align*}
We can now use Corollary~\ref{cor:risk-comp-esi} of
Theorem~\ref{thm:first} once again to get actual excess risk bounds
under the $v$-central condition (recall that the $\beta$-Bernstein
condition implies {$v$-central} with $v(\gamma) \asymp
\gamma^{1-\beta}$), plugging in the above expression
\eqref{eq:penalizedcompnew}.  We can expect this risk bound to be
tight since \eqref{eq:penalizedcompnew} is really an equality, and
Corollary~\ref{cor:risk-comp-esi}, the link between annealed and
actual excess risk for a given $v$-central condition, is also tight up
to constant factors. Of course, to make such a risk bound insightful
we would have to further bound $\log \shtark_{\eta}$, in a manner
similar as was done for ERM-like estimators in
Theorem~\ref{thm:comp-bound}.  Penalized empirical risk methods such
as the lasso have been thoroughly studied over the last fifteen years,
and we do not yet know whether the approach we just sketched will lead
to new results; our goal here is mainly to show that penalized ERM and
generalized Bayesian (randomized) estimators can both be analyzed
using the same technique, which bounds annealed risk in terms of
cumulative log-loss differences.

\paragraph{Cumulative Log-Loss Bounds --- Luckiness Regret} 
If the original loss function $\ell$ is log-loss and we take $\eta = 1$, then we can interpret the penalized estimator \eqref{eq:penalized} in terms of `minimax luckiness regret', 
which features prominently in  recent papers
on sequential individual sequence prediction with log-loss such as
\citep{kakade2006worst,bartlett2013horizon}, with the `luckiness' terminology introduced by \cite{grunwald2007the}: for
arbitrary probability densities (sequential log-loss prediction
strategies) $r$ on $\cZ^n$, we define the \emph{luckiness regret of $r$ on $z^n$ with slack function $\Gamma$} relative to set of densities $\{ p_f : f \in \cF \}$ as
\begin{align}\label{eq:minimax}
- \log r(z^n) - \min_{f \in \cF} \; \left(\; - \log p_{f}(z^n) + \Gamma(f)\; \right),
\end{align}
i.e.~the difference between the log-loss of $r$ and the log-loss
achieved by the $\Gamma$-penalized predictor $\hat{f}$ which minimizes the
penalized loss in hindsight.  Now, if we take luckiness function
$w(z^n) := \exp(-\Gamma(\hat{f}(z^n))$ and we take as $r$ in
\eqref{eq:minimax} the density $r_w$ for the penalized estimator
$\hat{f}$ as in definition \eqref{eq:nmldistb} (note that $p_f=q_f$,
since we work with log-loss), then from \eqref{eq:minimax} and that
definition we get that for each $z^n$, the
  luckiness regret of $r_w$ on $z^n$ is given by
\begin{align*}
& - \log r_{w}(z^n) - 
\min_{f \in \cF} \left(\; - \log p_{f}(z^n) + \Gamma(f)  \; \right) = \\ &
- \log p_{\hat{f}_{|z^n}}(z^n) + \Gamma(\hat{f}_{|z^n}) + \log \shtark(\cF,\hat{f},w) 
- 
\min_{f \in \cF} \left(\; - \log p_{f}(z^n) + \Gamma(f)  \; \right) = 
\log \shtark(\cF, \hat{f},w),
\end{align*}
so that the luckiness regret of $r_w$ is constant over $z^n$.  $r_w$
is thus an equalizer strategy, and, as explained by
\cite{grunwald2007the}, this implies that $r_w$ minimizes, over all
probability densities $r$, the maximum, over all $z^n \in \cZ^n$, of
\eqref{eq:minimax}, thus achieving the \emph{minimax luckiness regret}. This minimax luckiness regret is then also equal to
$\log \shtark(\cF,\hat{f},w)$.

\paragraph{Reconsidering  \cite{chatterjee2014information}}
Our first main result, the annealed risk convergence bound of
Theorem~\ref{thm:first}, when specialized to log-loss and
$\eta = 1/2$, implies a classic result of \cite{barron1991minimum}
that gives nonasymptotic Hellinger convergence rates for two-part MDL
estimators for well-specified models, implying that (two-part) data
compression implies learning. Such two-part MDL estimators invariably
work with a countable discretization of the parameter
space. \cite{chatterjee2014information} sought to use those bounds to
prove convergence at the right rate of Lasso-type estimators in a
Gaussian regression setting, showing that the $\ell_1$-penalization
can be linked to a minimization over a discretized grid of parameter
values that allows it to be related to two-part MDL so that the
Barron-Cover result can be used to prove rates of convergence.  The
present development suggests that this can perhaps be done much more
generally --- there is no need to consider only two-part codes or a
probabilistic setting: \emph{every} $\Gamma$-penalized estimator $\hat{f}$
for every bounded loss function defines a corresponding density $r_w$
with $w(x^n) = \exp(-\Gamma(\hat{f}))$, and hence a code with lengths
$-\log r_w(z^n)$ in terms of which one can prove an excess risk bound
via Theorem~\ref{thm:first} and Corollary~\ref{cor:risk-comp-esi}.

\section{Discussion and Future Work}
\label{sec:discussion}

Our strategy for controlling $\compnew(\F)$ owes much to an
ingenious argument of \cite{opper1999worst}. They analyzed the minimax
regret in the individual sequence prediction setting with log loss,
where the class of comparators is the set of static experts
(i.e.~experts that predict according to the same distribution in each
round).  \cite{cesa2001worst} obtain bounds in the more general
setting where the comparator class consists of \emph{arbitrary}
experts that can predict conditionally on the past (for a further
considerable extension within the realm of log loss, see
\cite{rakhlin2015sequential} who use sequential complexities).
Whereas the works of \cite{opper1999worst} and \cite{cesa2001worst}
both operate under some kind of bounded $L_\infty$ metric entropy (the
metric entropy in the latter work differs due to the experts'
sequential nature), the present paper operates under the much weaker
assumption of bounded $L_2$ metric entropy. We note, however, that
unlike the non-i.i.d.~setting of \cite{cesa2001worst}, the present
paper is restricted to the i.i.d./static experts setting. Yet, the
extension to general losses we introduce appears to be completely new.

Theorem~\ref{thm:excess-risk-erm} offers a distribution-dependent
bound whose derivation we view as simpler than similar bounds based on
local Rademacher complexities. In particular, the strategy adopted in
the present paper completely avoids complicated (at least in the view
of the authors) fixed point equations that have been used to obtain
good excess risk bounds in other works (such as
\cite{koltchinskii1999rademacher,bartlett2005local,koltchinskii2006local}).
In the case of classes of VC-type, one can obtain optimal rates by
decoupling the optimization of the parameters $\varepsilon$ and
$\gamma$; thus, one can obtain a suitable bound on $\compnew$ without
considering $\gamma$, leading to a rather easy tuning problem. In the
case of larger classes of polynomial empirical entropy or sets of
classifiers of polynomial entropy with bracketing, while $\gamma$
and $\varepsilon$ must be tuned jointly to obtain optimal rates, we
have shown that an optimal tuning can be obtained without great
effort.

We note, however, that the bounds in the present paper lack the kind
of data-dependence exhibited by previous works leveraging local
Rademacher complexities. Indeed, the bound in
Theorem~\ref{thm:excess-risk-erm} is an exact oracle inequality which
is distribution-dependent and, consequently, is not computable by a
practitioner who does not know the $\beta$ for which a Bernstein
condition holds. In contrast, bounds obtained via local Rademacher
complexities can be computed without distributional knowledge and have
been shown to behave like the correct (but unknown to the
practitioner) distribution-dependent bounds asymptotically (see
Theorem 4.2 of \cite{bartlett2005local}).

Yet, the present work gives rise to results which allow a different
kind of data-dependence: a PAC-Bayesian improvement for situations
when the posterior distribution is close to a prior distribution. This
improvement (which is also algorithm-dependent) is already apparent
from the simplified setting of Proposition~\ref{prop:decomposingcomp}
in which one places a prior over submodels, and we expect that much
more can be accomplished by using Theorem~\ref{thm:first} as a
starting point.

Theorem~\ref{thm:first} is also related to the main results of
\cite{audibert2007combining}, who provide bounds on the excess risk for bounded loss
functions that can involve the generic chaining technique of Fernique
and \cite{talagrand2014upper}; this technique generalizes the standard
chaining technique of Dudley and can lead to smaller complexities in
some cases.  To discuss the connection, first note that, as far as we
know, the standard chaining technique is used at some point in
\emph{all} approaches that achieve optimal rates for polynomial
entropy classes under Tsybakov or Bernstein conditions, although this
sometimes remains hidden\footnote{For example, the proofs of
  \cite{tsybakov2004optimal} are based on various results of
  \citet[Chapter 5]{vandegeer2000empirical} which are in turn based on
  chaining. We also note that if one strengthens the Bernstein
  condition to a two-sided version then, with $0/1$-loss, one can
  avoid chaining, see \cite{audibert2004PAC}.}.  In our approach,
chaining remains completely under the hood, but as mentioned earlier
it is present in the proof of Koltchinskii's
(\citeyear{koltchinskii2011oracle}) result linking Rademacher
complexities to empirical entropy.  Like we do,
\cite{audibert2007combining} provides bounds on excess risk that allow for the use of priors,
that can exploit Bernstein conditions, and that lead to optimal rates
for large classes. However, whereas in our work chaining remains under
the hood, their analogue of our `complexity' (the right-hand side of
their deviation bound) involves chaining explicitly, replacing the
KL-term by an infinite sum over (roots of) KL terms. This makes it
possible to design partitions of $\cF$ and priors thereon that allow
one to use generic chaining.  On the other hand, they directly bound
the excess risk --- there is no 'annealed' step in between and hence
no direct analogue of Theorem~\ref{thm:first} either --- so that it is
not clear whether their approach lends itself to the relatively easy
fixed-point-free tuning that is possible using our approach; also, the
Shtarkov integral and hence the connection to minimax log-loss regret
does not appear in their work, making the two approaches somewhat
orthogonal.

Thus, the aforementioned works go beyond our work in that they either
allow data-dependent analogues of Rademacher complexities (turning
oracle bounds into empirical bounds) or allow one to use generic
chaining; it is at this point unclear (and an interesting open
problem) whether our approach can be extended in these directions. We
stress, however, that these papers make no connection between excess
risk and NML complexity nor between NML complexity and Rademacher
complexities; these connections are, as far as we know, completely
new.

The recently developed notion of offset Rademacher complexity provides
a powerful alternative to analyses based on local Rademacher
complexities. \cite{liang2015learning} introduced offset Rademacher
complexities for the i.i.d.~statistical learning setting to obtain
faster rates under squared loss with unbounded noise (and hence
unbounded loss); their bounds hold for Audibert's star estimator
\citep{audibert2008progressive} --- an aggregation method --- and
obtain faster rates even in non-convex situations.  The techniques of
the present paper, while for general loss functions, notably do not
currently handle unbounded losses nor do they leverage aggregation; in
light of this latter trait, the rates obtained by
Theorem~\ref{thm:excess-risk-erm} in the case of squared loss with
non-convex classes are not minimax optimal as ERM itself fails to be
an optimal procedure \citep{juditsky2008learning}. On the other hand,
the rate provided by Theorem~\ref{thm:excess-risk-erm} is known to
tightly characterize the performance of ERM in a number of situations,
and it is unclear (to the authors) how to recover such results for ERM 
from the offset Rademacher complexity-based analysis of \cite{liang2015learning}.

\cite{zhivotovskiy2016localization} use a combination of offset Rademacher complexities with a shifted empirical process to obtain tight bounds for ERM for the case of classification with VC classes under Massart's noise condition. 
While in this setting our bounds are not as tight as those of \cite{zhivotovskiy2016localization}, our analysis applies to the case of general noise, general losses, and large classes. 
We note that in the case of classification and bounded noise, existing lower bounds imply that classes of infinite VC dimension fail to be learnable.

\newpage
\setlength{\glsdescwidth}{0.67\hsize}
\setlength{\glspagelistwidth}{0.05\hsize}
\renewcommand*{\pagelistname}{Page}
{ \small
\renewcommand*{\arraystretch}{1.2}% default is 1
\printglossary[style=long3colheaderborder,title=Glossary]
\label{glos:sary}
}

\newpage

% Acknowledgements should go at the end, before appendices and references

% \acks{}

\bibliography{nml_rad}

\newpage

\renewcommand{\theHsection}{A\arabic{section}}

\appendix

\glsadd{losses}
\glsadd{notation}
\glsadd{complexities}

\section{}
\label{app:proofs}

This section contains proofs omitted from the main text.
\subsection{Theorem~\ref{thm:first}}
\begin{proof}
Let us abbreviate 
$\ann{f} = n \Expann{\eta}_{\Zp \sim P} \left[ \xsloss{{f}}(\Zp) \right]$.
By the definition of ESI \eqref{eq:esi} 
we  see that the statement in the theorem is equivalent to
\begin{align}\label{eq:realthing}
\E_{Z^n \sim P} \left[ \exp \left( \eta \cdot \left( 
\E_{\rv{f} \sim \hat{\Pi}_{| Z^n}} \left[ \ann{\rv{f}} \right] 
- \fcompnew(\cF , \hat{\Pi}, w , Z^n) \right) 
\right)
\right] = 1.
\end{align}
Plugging in the definition of $\fcompnew$ and then $\compnew$, the left side can be rewritten as 
\begin{align*}
\E \left[ \exp \left( \eta \cdot \left( 
\E_{\rv{f} \sim \hat{\Pi}_{| Z^n}} \left[ 
\ann{\rv{f}}
-  R_{\rv{f}}(Z^n) \right] -
\frac{1}{\eta} \cdot \left(  
          \E_{\rv{f} \sim \hat{\Pi} \mid Z^n} 
\left[ - \log w(\rv{f}, Z^n) \right] 
          + \log \shtark(\cF, \hat{\Pi}, w) \right)
\right) \right)
\right] = \\
\frac{\E \left[ \exp \left( \eta \cdot \left( 
\E_{\rv{f} \sim \hat{\Pi}_{| Z^n}} \left[ 
\ann{\rv{f}}
-  R_{\rv{f}}(Z^n) \right] -
\frac{1}{\eta} \cdot  
          \E_{\rv{f} \sim \hat{\Pi} \mid Z^n} 
\left[ - \log w(\rv{f}, Z^n) \right] \right)\right)\right]} 
{
\E_{Z^n \sim P} \left[ \exp\left(- \E_{\rv{f} \sim \hat{\Pi} \mid z^n} \left[\eta 
\xsloss{\rv{f}}(z^n) + \log C(\rv{f}) -  \log w(z^n, \rv{f}) \right] \right) \right]
},         
\end{align*}
where the denominator is just the definition of $\shtark$. 
It is thus sufficient to prove that this expression is equal to $1$. But this is immediate from the definition of $C(f)$ and $\ann{\cdot}$. 
\end{proof}

\subsection{Proof of second main result, Theorem~\ref{thm:opper-haussler-talagrand}}
\label{app:opper-haussler-talagrand}

We first prove the results that imply \eqref{eqn:oht-1-first} and then prove the result that implies \eqref{eqn:oht-2-first}.

\subsubsection{Proof of \eqref{eqn:oht-1-first}}

Inequality \eqref{eqn:oht-1-first} from Theorem~\ref{thm:opper-haussler-talagrand} is a consequence of Lemmas~\ref{lemma:opper-haussler}~and~\ref{lemma:mostly-talagrand}, which we prove in turn.

\begin{Proof}[Proof of Lemma~\ref{lemma:opper-haussler}]
\begin{align*}
e^{\eta \cdot \compnew_\eta(\cF)} 
= \shtark(\F) 
&= \E_{\Zn \sim Q_{f_0}} \left[ 
              \sup_{f \in \F} \frac{q_f(\Zn)}{q_{f_0}(\Zn)}
          \right] \\
&= \E_{\Zn \sim Q_{f_0}} \left[ \exp \left( 
              \sup_{f \in \F} \log \frac{q_f(\Zn)}{q_{f_0}(\Zn)}
          \right) \right] \\
&\leq \E_{\Zn \sim Q_{f_0}} \left[ \exp \left( 
              \sup_{f \in \F} \left\{ 
                  \log \frac{q_f(\Zn)}{q_{f_0}(\Zn)} 
                  - \E_{\Zn \sim Q_{f_0}} \left[ \log \frac{q_f(\Zn)}{q_{f_0}(\Zn)} \right] 
              \right\} 
          \right) \right] ,
\end{align*}
where the inequality follows because the second term inside the supremum is a negative KL-divergence. 
Now, using the definition of $Q_f$ and $Q_{f_0}$, the above is equal to
\begin{align*}
    \E_{\Zn \sim Q_{f_0}} \left[ \exp \left( 
        \eta \underbrace{\sup_{f \in \F} \left\{ 
            \sum_{j=1}^n \left( \loss_{f_0}(Z_j) - \loss_f(Z_j) \right)
            - \E_{\Zn \sim Q_{f_0}} \left[ \sum_{j=1}^n \left( \loss_{f_0}(Z_j) - \loss_f(Z_j) \right) \right] 
        \right\} }_{T_n}
     \right) \right] .
\end{align*}
\end{Proof}

It remains to prove Lemma~\ref{lemma:mostly-talagrand}.

\begin{Proof}[Proof of Lemma~\ref{lemma:mostly-talagrand}]
First, from our assumption on the loss and $\eta \leq 1$ together imply that
\begin{align*}
\sup_{f,g \in \cF} \esssup \left\{ \eta \left( \loss_f(Z) - \loss_g(Z) - \E [ \loss_f(Z) - \loss_g(Z) ] \right) \right\} \leq 1 .
\end{align*}

Our goal now is to be able to apply Talagrand's inequality. 
To this end, observe that
\begin{align*}
\sup_{f,g \in \F} \mathsf{Var} \bigl[ \eta \left( \loss_f(Z) - \loss_g(Z) - \E [ \loss_f(Z) - \loss_g(Z) ] \right) \bigr] 
\leq \eta^2 \sup_{f,g \in \F} \| ( \loss_f - \loss_g ) \|_{L_2(Q_{f_0})}^2 .
\end{align*}

Now, \emph{if} $\F$ had a small $L_2(Q_{f_0})$ diameter, then the Lipschitzness of the loss would imply that the above term is also small. However, by assumption, the class $\F$ is only known to have small $L_2(P)$ diameter (of at most $\varepsilon$). 
Lemma~\ref{lemma:sigma} (stated after this proof) effectively bridges the gap between these two pseudonorms, showing that
\begin{align}
\sup_{f,g \in \F} \| \loss_f - \loss_g \|_{L_2(Q_{f_0})} 
\leq e \, \lip \sup_{f,g \in \F} \| f - g \|_{L_2(P)} , \label{eqn:apply-lemma-sigma}
\end{align}
which is then at most $e \, \lip \, \varepsilon = \sigma$.

Bousquet's version of Talagrand's inequality (see Theorem 2.3 of \cite{bousquet2002bennett} or, for a more direct presentation, Theorem 12.5 of \cite{boucheron2013concentration}) now yields
\begin{align*}
\E_{Q_{f_0}} [ e^{\lambda \eta T_{n,\eta}^{(k)}} ] 
\leq \exp \left( \E_{Q_{f_0}} [ \eta T_n^{(k)} ] + (e^\lambda - (\lambda + 1)) (n \eta^2 \sigma^2 + 2 \E_{Q_{f_0}} [ \eta T_n^{(k)} ] ) \right) .
\end{align*}
Inequality \eqref{eqn:talagrand-1} now follows by taking $\lambda = 1$. 
\end{Proof}

The following lemma was used to control the complexity of the class $\F$.
\begin{lemma} \label{lemma:sigma}
For the supervised loss parameterization,
\begin{align}
\| \loss_f - \loss_g \|_{L_2(Q_{f_0})} 
\leq e \cdot \lip \| f - g \|_{L_2(P)} . \label{eqn:switch-supervised}
\end{align}
For the direct parameterization,
\begin{align}
\| \loss_f - \loss_g \|_{L_2(Q_{f_0})} 
\leq e \| f - g \|_{L_2(P)} . \label{eqn:switch-direct}
\end{align}
\end{lemma}

\begin{Proof}[Proof of Lemma~\ref{lemma:sigma}]
We first prove \eqref{eqn:switch-supervised}, the supervised loss parameterization result. 
The Lipschitz assumption on the loss implies that
\begin{align*}
\E_{(X,Y) \sim Q_{f_0}} \left[ \left( \loss_f(X, Y) - \loss_g(X, Y) \right)^2 \right] 
\leq \lip^2 \E_{X \sim Q_{f_0}} \left[ \left( f(X) - g(X) \right)^2 \right] .
\end{align*}
Next, observe that for $\Delta(x) = \frac{q_{f_0}(x)}{p(x)}$
\begin{align*}
\E_{X \sim Q_{f_0}} \left[ \left( f(X) - g(X) \right)^2 \right] 
&= \E_{X \sim P} \left[ \Delta(x) \left( f(X) - g(X) \right)^2 \right] .
\end{align*}
Since the inside of the expectation is nonnegative, it remains to upper bound $\Delta(x)$.
By definition,
\begin{align*}
\Delta(x) 
\,=\, \frac{p(x) \int p(y \mid x) e^{-\eta \xsloss{f}(x,y)} dy}
           {p(x) \E_{(\bar{X},\bar{Y}) \sim P} \left[ e^{-\eta \xsloss{f}(\bar{X},\bar{Y})} \right]} 
\,=\, \frac{\E_{Y \sim P \mid X = x} \left[ e^{-\eta \xsloss{f}(x,Y)} \right]}
           {\E_{(\bar{X},\bar{Y}) \sim P} \left[ e^{-\eta \xsloss{f}(\bar{X},\bar{Y})} \right]} 
\,\leq\, e^\eta
\,\leq\, e ,
\end{align*}
since $\eta \leq 1$ and the excess loss random variable takes values in $[-1/2, 1/2]$.

We now prove the direct parameterization result \eqref{eqn:switch-direct}. 
Observe that for $\Delta(z) = \frac{q_{f_0}(z)}{p(z)}$
\begin{align*}
\E_{Z \sim Q_{f_0}} \left[ \left( \loss_f(Z) - \loss_g(Z) \right)^2 \right] 
&= \E_{Z \sim P} \left[ \Delta(Z) \left( f(Z) - g(Z) \right)^2 \right] ,
\end{align*}
where we use the fact that $\loss_f = f$ for all $f \in \F$ in the direct parameterization. As above, it remains to upper bound $\Delta(z)$.
By definition,
\begin{align*}
\Delta(z) 
\,=\, \frac{p(z) e^{-\eta \xsloss{f}(z)}}
           {p(z) \E_{\Zp \sim P} \left[ e^{-\eta \xsloss{f}(\Zp)} \right]} 
\,\leq\, e^\eta
\,\leq\, e .
\end{align*}
\end{Proof}

\subsubsection{Proof of \eqref{eqn:oht-2-first}}

Inequality~\eqref{eqn:oht-2-first} is a consequence of \eqref{eqn:oht-1-first} and a standard empirical process theory result, Lemma~\ref{lemma:E-sup-rad}. For completeness, we provide a proof of this result below.

\begin{Proof}[Proof of Lemma~\ref{lemma:E-sup-rad}]
Recall that $\G = \{ \loss_{f_0} - \loss_f : f \in \F \}$, and let $\epsilon_1, \ldots \epsilon_n$ be independent Rademacher random variables. 
In the below, both $\Zn$ and $\Zpn$ are drawn from $Q_{f_0}$. 

The following sequence of inequalities is a standard use of symmetrization from empirical process theory:
\begin{align*}
&\E \left[ \sup_{f \in \F} \left\{ 
            \sum_{j=1}^n \left( \loss_{f_0}(Z_j) - \loss_f(Z_j) \right)
            - \E \left[ \sum_{j=1}^n \left( \loss_{f_0}(\Zp_j) - \loss_f(\Zp_j) \right) \right] 
        \right\} \right] \\
&= \E \left[ \sup_{g \in \G} \left\{ 
            \sum_{j=1}^n g(Z_j) 
            - \E \left[ \sum_{j=1}^n g(\Zp_j) \right] 
        \right\} \right]  \\
&\leq \E \left[ \sup_{g \in \G}  
            \sum_{j=1}^n \left( g(Z_j) - g(\Zp_j) \right)  
        \right] \\
&= \E \left[ \sup_{g \in \G}  
            \sum_{j=1}^n \epsilon_j \left( g(Z_j) - g(\Zp_j) \right)  
      \right] \\
&\leq 2 \E \left[ \sup_{g \in \G}  
                 \sum_{j=1}^n \epsilon_j g(Z_j)  
             \right] \\
&\leq 2 \E \left[ \sup_{g \in \G}  
                 \left| \sum_{j=1}^n \epsilon_j g(Z_j) \right|  
             \right] .
\end{align*}
\end{Proof}

\subsection{Proof of Theorem~\ref{thm:comp-bound}}
\label{app:comp-bound}

\begin{Proof}[Proof of Theorem~\ref{thm:comp-bound}]
Taking the results of Corollary~\ref{cor:opper-haussler-talagrand} and dividing by $n$ gives the two inequalities
\begin{align}
\frac{\compnew_\eta(\F)}{n} 
\leq \frac{\log \N(\F, L_2(P), \varepsilon/2)}{n \eta} 
        + \frac{3}{n} \max_{k \in [N_\varepsilon]} \E_{\Zn \sim Q_{f_k}} \left[ T_n^{(k)} \right]
        + \eta \sigma^2 \label{eqn:pre-full-bound-1}
\end{align}
and
\begin{align}
\frac{\compnew_\eta(\F)}{n} 
\leq \frac{\log \N(\F, L_2(P), \varepsilon/2)}{n \eta} 
        + 6 \max_{k \in [N_\varepsilon]} \E_{\Zn \sim Q_{f_k}} \left[ \rad_n(\G_k) \right]
        + \eta \sigma^2 , \label{eqn:pre-full-bound-2}
\end{align}
where we remind the reader that $N_\varepsilon = \N(\F, L_2(P), \varepsilon/2)$. 

In the below applications of Theorems~\ref{thm:small-rad} and \ref{thm:small-Esup}, we make use of the following two observations. 
First, from Lemma~\ref{lemma:sigma} (which we previously applied to yield \eqref{eqn:apply-lemma-sigma}), it follows that the $L_2(Q_{f_k})$ diameter of $\G_k$ is at most $\sigma$. 
Second, for any distribution $Q \in \Delta(\Z)$, for all $u > 0$,
\begin{align}
\N(\G_k, L_2(Q), u) 
= \N(\{ \loss_f : f \in \Fpart{k} \}, L_2(Q), u) 
\leq \N(\Fpart{k}, L_2(Q), u / \lip) \label{eqn:Gk-L2Pn}
\end{align}
and (in the case of sets of classifiers)
\begin{align}
N_{[\cdot]}(\G_k, L_2(Q_{f_k}), u) 
= \N_{[\cdot]}(\{ \loss_f : f \in \Fpart{k} \}, L_2(Q_{f_k}), u) 
&= \N_{[\cdot]}(\Fpart{k}, L_2(Q_{f_k}), u) \label{eqn:Gk-L2Q-b} \\
&\leq \N_{[\cdot]}(\Fpart{k}, L_2(P), u/e) ; \nonumber 
\end{align}
in both \eqref{eqn:Gk-L2Pn} and \eqref{eqn:Gk-L2Q-b}, the first equality holds because $\G_k$ is a shifted version of $\{\loss_f : f \in \F\}$. In the case of the supervised loss parameterization, the inequality in \eqref{eqn:Gk-L2Pn} holds from the Lipschitzness of the loss, and, in the case of the direct parameterization, the inequality is actually equality (recall that $L= 1$ in this case). 
The second equality of \eqref{eqn:Gk-L2Q-b} holds because we only consider sets of classifiers with 0-1 loss. Lastly, the inequality in \eqref{eqn:Gk-L2Q-b} is due to the 1-Lipschitzness of 0-1 loss for sets of classifiers and Lemma~\ref{lemma:sigma}. 
From \eqref{eqn:Gk-L2Pn}, if $\F$ is a VC-type class (and hence so is $\Fpart{k}$), then $\G_k$ also is a VC-type class. Analogously, if $\F$ has polynomial empirical entropy, the same property extends to $\G_k$. 
From \eqref{eqn:Gk-L2Q-b}, if $\F$ is a class whose $L_2(P)$ entropy with bracketing is polynomial (and hence so is $\Fpart{k}$), then $\G_k$ is a class whose $L_2(Q_{f_k})$ entropy with bracketing is polynomial with the same exponent.

\paragraph{VC-type classes.}
First, Theorem~\ref{thm:L2Pn-to-L2P} (stated after this proof) implies that, for all $u > 0$,
\begin{align*}
\N(\F, L_2(P), u) \leq \left( \frac{2 A}{u} \right)^V .
\end{align*}
Starting from \eqref{eqn:pre-full-bound-2}, inequality \eqref{eqn:vc-rad-bound} from Theorem~\ref{thm:small-rad}
combined with \eqref{eqn:Gk-L2Pn}
then implies that (coarsely using $\eta \leq 1$)
\begin{align*}
\frac{\compnew_\eta(\F)}{n} 
&\lesssim \frac{V \log \frac{4 A}{\varepsilon}}{n \eta} 
 + \max \left\{ \sqrt{\frac{V}{n}} \sigma \sqrt{\log \frac{A \lip}{\sigma}},
                \frac{V U}{n} \log \frac{A \lip}{\sigma} \right\} 
 + \eta \sigma^2 \\
&\lesssim \frac{V \log \frac{4 A}{\varepsilon}}{n \eta} 
      + \max \left\{ \sqrt{\frac{V}{n}} \lip \, \varepsilon \sqrt{\log \frac{A}{\varepsilon}},
                     \frac{V U}{n} \log \frac{A}{\varepsilon} \right\} 
      + (\lip \, \varepsilon)^2 .
\end{align*}
Finally, setting $\varepsilon = \frac{4}{\lip} \sqrt{\frac{V}{n}}$ yields (up to a universal multiplicative constant) the bound
\begin{align*}
\frac{V \log \frac{A \, \lip \, n}{V}}{n \eta} 
+ \max \left\{ \frac{V}{n} \sqrt{\log \frac{A \lip n}{V}},
               \frac{V}{n} \log \frac{A \lip n}{V} \right\} 
+ \frac{V}{n} 
\,\lesssim\, \frac{V \log \frac{A \, \lip \, n}{V}}{n \eta} ,
\end{align*}
where we used the assumption that $\eta \leq 1$. This proves \eqref{eqn:comp-vc}.

\paragraph{Classes of polynomial empirical entropy or polynomial entropy with bracketing.}

The first order of business is to control $N_\varepsilon = \log \N(\F, L_2(P), \varepsilon/2)$. In the case of classes of polynomial empirical entropy, we again invoke Theorem~\ref{thm:L2Pn-to-L2P} to conclude that, for all $u > 0$,
\begin{align*}
\log \N(\F, L_2(P), u) \leq \left( \frac{2 A}{u} \right)^{2 \rho} .
\end{align*}

In the case of sets of classifiers of polynomial entropy with bracketing, the $L_2(P)$ entropy can be controlled by the relationship
\begin{align*}
\log \N(\F, L_2(P), u) 
\leq \log \N_{[\cdot]}(\F, L_2(P), u) 
= \log \N_{[\cdot]}(\F, L_1(P), u^2) 
\leq \left( \frac{A}{u} \right)^{2 \rho} .
\end{align*}
Next, for \emph{(i)} classes of polynomial empirical entropy, we start from \eqref{eqn:pre-full-bound-2} and apply inequality \eqref{eqn:poly-rad-bound} from Theorem~\ref{thm:small-rad} combined with \eqref{eqn:Gk-L2Pn}; or \emph{(ii)} for classes of polynomial entropy with bracketing, we start from \eqref{eqn:pre-full-bound-1} and apply\footnote{Note that in classification, for any $Q$, the $L_1(P)$ diameter is equal to the square of the $L_2(P)$ diameter.} Theorem~\ref{thm:small-Esup} combined with \eqref{eqn:Gk-L2Q-b}; both cases imply that, for $0 < \eta \leq 1$, using $\rho < 1$, 
\begin{align}
\frac{\compnew_\eta(\F)}{n} 
&\lesssim \frac{1}{n \eta} \left( \frac{2 A}{\varepsilon} \right)^{2 \rho} 
 + \max \left\{ \frac{(A \lip)^\rho}{\sqrt{n}} \sigma^{1 - \rho}, 
                \frac{(A \lip)^{2 \rho / (\rho + 1)} U^{(1 - \rho) / (1 + \rho)}}{n^{1 / (1 + \rho)}} \right\} 
 + \eta \sigma^2 \nonumber \\ \label{eqn:comp-poly-pre}
&\lesssim \frac{1}{n \eta} \left( \frac{A}{\varepsilon} \right)^{2 \rho} 
          + \frac{A^\rho \lip}{\sqrt{n}} \varepsilon^{1 - \rho} +  
                                 \eta^{\frac{\rho-1}{\rho+1}} \cdot \frac{(A \lip)^{2 \rho / (\rho + 1)}}{n^{1 / (1 + \rho)}} 
          + \eta \cdot (\lip \, \varepsilon)^2 .
\end{align}
(the enlargement of the third term will not affect the rates, as will now become clear).
We now set
$\epsilon := C_0 n^{-\frac{1}{2 (1 + \rho)}} \cdot \eta^{-\frac{1}{1+ \rho}}$
for a constant $C_0 > 0$ to be determined later (this choice for
$\epsilon$ was obtained by minimizing the sum of the first and second terms in the last line of \eqref{eqn:comp-poly-pre} by setting the derivative to $0$).
With this choice, we get, as a very simple yet tedious calculation shows:
\begin{align*}
n^{-1} \eta^{-1} \epsilon^{- 2 \rho} & 
= C_0^{-2 \rho} \cdot n^{-\frac{1}{1 + \rho}} \cdot \eta^{\frac{\rho-1}{\rho+1}} \\
n^{-1/2} \epsilon^{1- \rho} & = C_0^{1- \rho} \cdot n^{-\frac{1}{1 + \rho}} \cdot \eta^{\frac{\rho-1}{\rho+1}} \\
\eta \epsilon^{2} & = C_0^{2} \cdot n^{-\frac{1}{1 + \rho}} \cdot \eta^{\frac{\rho-1}{\rho+1}} \\
\end{align*}
so that  \eqref{eqn:comp-poly-pre} becomes 
\begin{align}
\frac{\compnew_\eta(\F)}{n} \lesssim C_{A,C_0,\lip} \cdot n^{-\frac{1}{1 + \rho}} \cdot \eta^{\frac{\rho-1}{\rho+1}}
\end{align}
where
\begin{align}
C_{A,C_0,\lip} = \left( \frac{2 A}{C_0} \right)^{2 \rho}  
  + {A^\rho \lip^{\rho}} (C_0 \, e )^{1 - \rho} 
  + {(A \lip)^{2 \rho / (\rho + 1)}} 
  + (e \, \lip \, C_0)^2.
\end{align}
Plugging in $C_0 =A^{\rho/(\rho+1)} L^{-1/(\rho +1)}$, the four terms become of the same order:
\begin{align*}
C_{A,C_0,\lip} & \lesssim  \left(L^{1/(\rho+1)}  A^{1 - \frac{\rho}{\rho+1}} \right)^{2 \rho}  
  + L^{1 - \frac{1-\rho}{1+\rho}}  A^{\rho + \frac{\rho (1- \rho)}{\rho+1}}
  + {(A \lip)^{2 \rho / (\rho + 1)}} 
  + (L^{1 - \frac{1}{1+ \rho}} A^{\frac{\rho}{\rho+1}})^2 \\
& \lesssim  (A \lip)^{2 \rho / (\rho + 1)},
\end{align*}
and \eqref{eqn:comp-poly-newbee} follows. 
\end{Proof}

The above proof made use of the universal $L_2(P)$ metric entropy being essentially equivalent to the universal $L_2(P_n)$ metric entropy. This result extends an analogous result of \cite{haussler1995sphere} for VC classes (see Corollary 1 therein). 

\begin{theorem}[Extended Haussler] \label{thm:L2Pn-to-L2P} 

Let $\F$ be a class of functions over a space $\cS$. 
Suppose that, for all $\varepsilon > 0$ and all $n \in \mathbb{N}$, there is some function $\psi \colon \reals_+ \rightarrow \mathbb{N}$ such that
\begin{align*}
\sup_{s_1, \ldots, s_n \in \cS} \N(\F, L_2(P_n), \varepsilon) \leq \psi(\varepsilon) .
\end{align*}
Then, for any probability measure $P \in \Delta(\cS)$ and any $\varepsilon > 0$,
\begin{align*}
\N(\F, L_2(P), \varepsilon) \leq \psi(\varepsilon/2) .
\end{align*}
\end{theorem}
The proof is essentially due to Haussler with little change to the argument for the more general result.

\begin{Proof}[Proof of Theorem~\ref{thm:L2Pn-to-L2P}]
Let $d$ be some pseudometric on $\F$. 
We say that $U \subset \F$ is $\varepsilon$ separated if, for all $f, g \in U$, it holds that $d(f, g) > \varepsilon$. 
Let the $\varepsilon$-packing number $\M(\F, d, \varepsilon)$ be the maximal size of an $\varepsilon$-separated set in $\F$.

The packing numbers and covering numbers satisfy the following relationship \citep[Lemma 2.2]{vidyasagar2002learning}
\begin{align*}
\M(\F, d, \varepsilon) \leq \N(\F, d, \varepsilon/2) .
\end{align*}

Thus, it is sufficient to bound $\M(\F, L_2(P), \varepsilon)$. 

Suppose that $\M(\F, L_2(P), \varepsilon) > \M(\F, L_2(P_n), \varepsilon)$, and take $U$ to be some $\varepsilon$-separated subset of $\F$ in the $L_2(P)$ pseudometric of cardinality $|U| > \M(\F, L_2(P_n), \varepsilon)$. 
 
Next, draw $s_1, \ldots, s_n$ i.i.d.~from $P$. Since $U$ is finite, by taking $n$ large enough we can ensure that the event $A_{f,g}$, defined as, 
\begin{align*}
\|f - g\|_{L_2(P_n)} = \left( \frac{1}{n} \sum_{j=1}^n (f(s_j) - g(s_j)) \right)^{1/2} < \varepsilon ,
\end{align*}
occurs with probability at most $\frac{1}{|U|^2}$. Since ${|U| \choose 2} < |U|^2$, it follows that the probability that no event $A_{f,g}$ occurs among all $f,g \in U$ is positive.
Hence, there exists a set of points $s_1, \ldots, s_n$ for which $U$ is an $\varepsilon$-packing in the $L_2(P_n)$ pseudometric. But then it must be the case that $\M(\F, L_2(P_n), \varepsilon) \geq |U|$, contradicting our assumption that $|U| > \M(\F, L_2(P_n), \varepsilon)$.
\end{Proof}

\subsection{Proofs for Section~\ref{sec:applications}}
\subsubsection{Proof of Proposition~\ref{prop:maxcomp}}
Using $w(f,z^n) \equiv 1$, we can write: 
{\begin{align*}
\compnew(\cF, \hat{\Pi}, w,z^n) &= 
\frac{1}{\eta} \log \E_{\Zn \sim P} \left[ 
\exp\left( -
\E_{\rv{f} \sim \hat{\Pi} \mid Z^n} \left[
\eta  \xsloss{\rv{f}}(\Zn)  + \log C(\rv{f}) \right] \right) \right] 
  \\
  &\leq \frac{1}{\eta} \log \E_{\Zn \sim P} \left[ \sup_{f \in \cF} \frac{e^{-\eta
        \xsloss{{f}}(\Zn)}}{C(f)} \right] ,
\end{align*}} 
which is just $\compnew(\cF)$.
\subsubsection{Proof of Proposition~\ref{prop:decomposingcomp}}
Plugging the definition $w$ into the definition of $\compnew$, a 
sequence of straightforward rewritings gives: 
\begin{align}\label{eq:nazomer}
& \compnew(\cF, \hat{\Pi}, w, z^n)  \nonumber\\
& = \frac{1}{\eta} \cdot 
\left(  
\E_{\rv{k} \sim \hat{\Pi}_{\cK} \mid z^n} 
\E_{\rv{f} \sim \hat{\Pi} \mid \rv{k}, z^n} 
\left[
\log \frac{ \hat{\pi}_{\cK}(\rv{k} \mid z^n)  }
{\pi_{\cK}(\rv{k}) \cdot  w_{\rv{k}}(z^n,\rv{f}) }  \right]
+ \log \shtark(\cF, \hat{\Pi}, w) \right) \nonumber \\
& = \frac{1}{\eta} \cdot 
\left(  
\E_{\rv{k} \sim \hat{\Pi}_{\cK} \mid z^n} \left[ \log \frac{ \hat{\pi}_{\cK}(\rv{k} \mid z^n)
}{\pi_{\cK}(\rv{k})} + 
\E_{\rv{f} \sim \hat{\Pi} \mid \rv{k}, z^n} 
\left[
- \log w_{\rv{k}}(z^n, \rv{f})  
\right] \right]
+ \log \shtark(\cF, \hat{\Pi}, w) \right) \nonumber \\
& = \frac{1}{\eta} \cdot \KL(\; (\hat{\Pi}_{\cK} \mid z^n) \pipes \Pi_{\cK}) 
       + \E_{\rv{k} \sim \hat{\Pi}_{\cK} \mid z^n} \E_{\rv{f} \sim \hat{\Pi} \mid \rv{k}, z^n} 
\left[ - \log w_{\rv{k}}(z^n, \rv{f})   \right]
+ \frac{1}{\eta} \log \shtark(\cF, \hat{\Pi}, w) .
\end{align}
If we can further show that
\begin{align}\label{eq:herfst}
\log \shtark(\cF, \hat{\Pi}, w) \leq
\E_{\rv{k} \sim \hat{\Pi}_{\cK} \mid z^n} \left[
\log \shtark(\cF_{\rv{k}}, \hat{\Pi}_{|\rv{k}}, w_{\rv{k}})
 \right]
\end{align}
then the result follows by plugging this into the last line of 
\eqref{eq:nazomer}. We thus proceed to show \eqref{eq:herfst}.
{Setting 
\begin{align*}
g(k,z'^n) =
\E_{\rv{f} \sim \hat{\Pi} \mid {k}, z'^n} 
\left[ \eta R_{\rv{f}} + \log C(\rv{f}) - \log w_{{k}}(z'^n, \rv{f})   \right]
, 
\end{align*}}
we can write:
\begin{align*}
&\log \shtark(\cF, \hat{\Pi}, w) 
= \log \E_{\Zn \sim P} \left[
\exp\left( 
\E_{\rv{k} \sim \hat{\Pi}_{\cK} \mid z^n}\left[ 
 \log  \frac
{\pi_{\cK}(\rv{k})}{ \hat{\pi}_{\cK}(\rv{k} \mid z^n)  }
\right] \right) \cdot 
\exp\left(- 
\E_{\rv{k} \sim \hat{\Pi}_{\cK} \mid z^n} \left[ g(\rv{k},Z^n) \right]
\right) \right]
\\& \leq \log \E_{\Zn \sim P} \left[
\left( 
\E_{\rv{k}  \sim \hat{\Pi}_{\cK} \mid z^n}\left[ 
\frac
{\pi_{\cK}(\rv{k})}{ \hat{\pi}_{\cK}(\rv{k} \mid z^n)  }
\right] \right) \cdot 
\exp\left(- 
\E_{\rv{k} \sim \hat{\Pi}_{\cK} \mid z^n} \left[ g(\rv{k},Z^n) \right]
\right) \right]
\\
\\& =  \log \E_{\Zn \sim P} \left[
\exp\left(- 
\E_{\rv{k} \sim \hat{\Pi}_{\cK} \mid z^n} \left[ g(\rv{k},Z^n) \right]
\right) \right]
\\
&\leq
\E_{\rv{k} \sim \hat{\Pi}_{\cK} \mid z^n}\left[
\log \E_{\Zn \sim P} \left[
\exp\left(- g(\rv{k},Z^n) 
\right)
\right] \right] = \E_{\rv{k} \sim \hat{\Pi}_{\cK} \mid z^n}\left[
\log \shtark(\cF, \hat{\Pi}_{|\rv{k}}, w_{\rv{k}}) \right]
\end{align*}
where the first and last equalities are just definition chasing, the first inequality
is Jensen's and the lastinequality is Lemma 3.2. from
\cite{audibert2009fast}; the result follows.
\end{document}